\documentclass{article}
\usepackage{log_2026}						

\usepackage{booktabs}						
\usepackage{multirow}						
\usepackage{amsfonts}						
\usepackage{graphicx}						
\usepackage[sort,round]{natbib}			

\usepackage{amsmath}
\usepackage{amssymb}
\usepackage{mathrsfs}
\usepackage{amsthm}
\usepackage{comment}
\usepackage{tabularx}
\usepackage{array}
\usepackage{tcolorbox}
\newcolumntype{Y}{>{\raggedright\arraybackslash}X}
\theoremstyle{plain}
\newtheorem{theorem}{Theorem}[section]
\newtheorem{proposition}[theorem]{Proposition}
\newtheorem{lemma}[theorem]{Lemma}
\newtheorem{corollary}[theorem]{Corollary}

\definecolor{deeppurple}{HTML}{673AB7}
\newcommand{\RB}[1]{\textcolor{deeppurple}{[\textbf{RB:} #1]}}

\newcommand{\rev}[1]{#1}

\newcommand{\rone}[1]{{\color{red}\textbf{#1}}}
\newcommand{\rtwo}[1]{{\color{purple}\textbf{#1}}}
\newcommand{\rthree}[1]{{\color{blue}\textbf{#1}}}

\title[
Fixed Points Without Fixed Diffusion
]{
Fixed Points Without Fixed Diffusion: Implicit Neural Sheaves for Convergent Test-Time Computation
}

\author[R.~Bourgerie et al.]{%
R{\'e}mi Bourgerie\And
\v{S}ar\={u}nas Girdzijauskas\And
Viktoria Fodor\AND
\normalfont\small School of Electrical Engineering and Computer Science, and Digital Futures\\
\normalfont\small KTH Royal Institute of Technology, Stockholm, Sweden\\
\normalfont\small\email{\{remibo,sarunasg,vfodor\}@kth.se}
}

\begin{document}

\maketitle

\begin{abstract}

Implicit Graph Neural Networks (IGNNs) define node representations as fixed points of \rev{graph} neural operators, enabling effectively infinite-depth propagation, iteration-independent parameterization, and flexible test-time computation.
 Yet these benefits depend on the equilibrium being unique and \rev{reached} by fixed-point iteration. 
 Existing constructions often impose constraints on recurrent updates to obtain these guarantees, limiting the transformations available at equilibrium. This raises a central question: can IGNNs gain expressiveness through richer, edge-dependent transformations while retaining the inherent strengths of their equilibrium formulation?
We introduce SheafDEQ, a subhomogeneous deep-equilibrium architecture with adaptive neural-sheaf propagation. 
Its learned, matrix-valued sheaf restriction maps can align, mix, or reverse neighbouring representations. Under mild regularity conditions, we prove that SheafDEQ admits a unique equilibrium reached globally by fixed-point iteration from any positive \rev{initial state}.
Contractivity further guarantees convergence under bounded communication staleness. 
We evaluate SheafDEQ on distributed-inference tasks requiring repeated
nonlocal aggregation and on community detection whose rewiring increasingly
favours cross-community interactions.
SheafDEQ improves over implicit baselines on Sums, MNIST
Terrain, and Coordinates, and on community detection as connectivity becomes
increasingly heterophilic. \rev{SheafDEQ remains competitive with recurrent
models without equilibrium guarantees.} Continued-iteration diagnostics show decreasing
residuals \rev{and retained test accuracy through 100 iterations, unlike the
finite-horizon control}, while delayed-update experiments
show low sensitivity to bounded communication staleness.



\end{abstract}

\section{Introduction}
\label{sec:intro}

Implicit Graph Neural Networks (IGNNs)~\citep{scarselli2008graph,gu2020implicit,yang2025implicit}
define node representations as fixed points of graph neural operators. Given
node features $X$ on a graph $G$, a $K$-layer explicit GNN composes
layer-specific updates, whereas an implicit GNN seeks an equilibrium
\[
H^\star=\mathcal T_\Theta(H^\star;X,G),
\]
with $\Theta$ shared across iterations. A natural inference procedure repeatedly
applies the update
\(
H^{(k+1)}=\mathcal T_\Theta(H^{(k)};X,G).
\) This formulation enables effectively infinite-depth propagation with an
iteration-independent parameter count and flexible test-time
computation~\citep{bai2019deep,gu2020implicit,yang2025implicit}. Under suitable
contractivity conditions, the \rev{iterates} can also converge under asynchronous
execution and bounded communication
delays~\citep{solodova2025graph,bourgerie2026euclidean}.

However, the fixed-point equation alone does not ensure that this iteration converges.
Existence and uniqueness of the equilibrium are likewise not automatic.
Existing implicit GNNs obtain the required guarantees through constrained
recurrent constructions: by bounding recurrent operators in norm or spectral
radius~\citep{gu2020implicit,liu2021eignn,park2021convergent}, imposing strongly
monotone or orthogonal parameterizations~\citep{baker2023implicit}, or minimizing
strongly convex learned energies~\citep{solodova2025graph}. While mathematically effective, these
mechanisms narrow the class of admissible learnable recurrent dynamics by
construction.

Subhomogeneous deep equilibrium (SubDEQ) models provide a different route to
these guarantees. They derive contractivity from positive subhomogeneous
nonlinearities rather than constraints on the learned recurrent weights,
yielding a unique equilibrium reached by fixed-point iteration while leaving
these weights spectrally unconstrained~\citep{sittoni2024subhomogeneous}. Their
graph realization, however, uses a fixed normalized adjacency
operator~\citep{sittoni2024subhomogeneous}.
The missing capability is therefore an implicit GNN that retains convergent
test-time computation while learning edge-dependent transformations from the
current \rev{state}. For example, long-range and distributed tasks often
require repeated propagation of nonlocal information~\citep{solodova2025graph},
while heterophilous interactions can benefit from edge-dependent
transformations rather than scalar averaging
~\citep{bodnar2022neural,bourgerie2026deep}.

Graph attention is the canonical adaptive alternative to fixed propagation: it
infers normalized nonnegative scalar edge weights from the current node
representations~\citep{velickovic2018graph,brody2021attentive}, but neighbouring
messages are still combined through scalar convex averaging, which can drive
representations toward uniformity under repeated
application~\citep{wu2023demystifying}. Sheaf propagation instead uses signed,
matrix-valued edge interactions that can align, mix, or reverse messages before
aggregation~\citep{curry2014sheaves,hansen2019toward,bodnar2022neural}. Suitable
sheaves can therefore retain structured, nonuniform signals in the
limit~\citep{bodnar2022neural}.

Neural sheaf networks make this propagation adaptive by inferring restriction
maps from the current representations~\citep{bodnar2022neural}. Unlike a
finite-depth neural sheaf network~\citep{bodnar2022neural,bourgerie2026deep}, an
implicit model repeatedly applies the same shared update: the \rev{state}
determines the restriction maps, the maps determine the sheaf propagation
operator, and this operator acts again on the \rev{state}. This feedback
does not automatically preserve the homogeneity and regularity required by
SubDEQ theory. Neither fixed-sheaf diffusion
~\citep{hansen2019toward,bodnar2022neural} nor fixed-propagation graph SubDEQ
~\citep{sittoni2024subhomogeneous} directly covers
\rev{an equilibrium with state-dependent propagation}.

We address this challenge with \emph{SheafDEQ}, a subhomogeneous equilibrium
model with \rev{adaptive} neural sheaf propagation. Its central
construction makes restriction-map inference scale-invariant: scaling the
current \rev{state} leaves the inferred maps unchanged, so sheaf propagation
remains positively homogeneous despite its state dependence. Combining this
positive homogeneity with a positive subhomogeneous nonlinear update allows the
SubDEQ contraction guarantees to apply to the complete update.

Our contributions are:
\begin{itemize}
    \item We introduce \emph{SheafDEQ}, an implicit GNN with adaptive, signed,
    and matrix-valued sheaf propagation whose scale-invariant restriction-map
    inference preserves positive homogeneity, and show in a fixed-sheaf setting
    that this propagation can yield class-separating positive equilibria.

    \item Under explicit regularity conditions, we prove global linear
    fixed-point convergence from every entrywise-positive \rev{initial state} to a
    unique equilibrium. Under bounded update and communication delays,
    asynchronous iterations also converge to this equilibrium, without
    spectral-radius or entrywise sign constraints on the learned restriction
    maps.

    \item \rev{We evaluate SheafDEQ on distributed-inference and
    community-detection tasks, demonstrating competitive predictive performance
    alongside convergence diagnostics and robustness to bounded communication
    staleness.}
\end{itemize}
Section~\ref{sec:background} reviews implicit GNNs, SubDEQ theory, and neural
sheaves. Section~\ref{sec:method} presents SheafDEQ and proves synchronous
and asynchronous convergence, and Section~\ref{sec:experiments} evaluates
predictive performance, finite-iteration recurrence diagnostics, and
sensitivity to bounded communication staleness. \rev{Notation is listed in
Appendix~\ref{app:notation}.}

\section{Background and Related Work}
\label{sec:background} 

\paragraph{Implicit GNNs.}
Given an undirected graph $G=(V,E)$ with $n=|V|$ nodes, edge set
$E\subseteq\{\{u,v\}:u,v\in V,\ u\neq v\}$, and $f$-dimensional node
features $X\in\mathbb R^{n\times f}$, let
$H^{(k)}\in\mathbb R^{n\times c}$ denote the \rev{recurrent state} at
iteration $k$. An implicit GNN defines its output as the equilibrium of a
graph neural operator~\citep{scarselli2008graph,gu2020implicit,yang2025implicit},
\[
    H^\star
    =
    \mathcal T_\Theta(H^\star;X,G).
\]
The equilibrium is \emph{well-posed} when it exists and is unique for every
admissible $(X,G)$. Well-posedness alone does not ensure that a chosen
inference procedure reaches $H^\star$. Picard iteration, for example,
repeatedly applies
$H^{(k+1)}=\mathcal T_\Theta(H^{(k)};X,G)$ until the iterates stabilize, and
its convergence requires additional conditions. Existing implicit GNNs
address these requirements through operator-norm or spectral-radius
bounds~\citep{gu2020implicit,liu2021eignn,park2021convergent},
monotone-operator parameterizations~\citep{baker2023implicit}, or strongly
convex learned energies~\citep{solodova2025graph}.
Appendix~\ref{app:implicit-gnns} details the associated equilibrium, inference,
and differentiation guarantees.

\paragraph{Subhomogeneous deep equilibrium models.} Many well-posed implicit
models obtain their guarantees by constraining the recurrent
weights~\citep{gu2020implicit,liu2021eignn,baker2023implicit}.
Subhomogeneous Deep Equilibrium Models (SubDEQs) instead move the sufficient
condition from the weights to the nonlinearities, leaving the recurrent weight
matrices unrestricted
~\citep{sittoni2024subhomogeneous}. A Lipschitz map
\(T:\mathbb R^a\to\mathbb R^b\) is \(\mu\)-\emph{subhomogeneous} on
\(\Omega\subseteq\operatorname{dom}_{+}(T):=\{z:T(z)\geq0\}\) if
\[
|Mz|\leq\mu T(z),\qquad z\in\Omega,\quad M\in\partial T(z),
\]
where \(\partial T(z)\) is the Clarke generalized Jacobian and the inequality
is entrywise. If a positive self-map
\(T:\mathbb R_{++}^a\to\mathbb R_{++}^a\) is \(\mu\)-subhomogeneous with
\(\mu<1\), then it is contractive in the Thompson metric
\(
\delta_{\mathrm T}(x,y):=\|\log x-\log y\|_\infty,
\)
and Picard iteration converges globally and linearly to its unique fixed
point~\citep{sittoni2024subhomogeneous}.
SubDEQs use these results to design recurrent models of the form
\[
\mathcal T_{W,\theta}(H;X)=\sigma_1\!\left(\sigma_2(H\star W)+B_\theta(X)\right),
\]
where \(W\) is learned without sign or norm constraints, \(B_\theta(X)\) is
entrywise positive, and the nonlinearities make the complete update positive
and sufficiently subhomogeneous. Taking \(\sigma_1=\mathrm{Id}\) and
\(\sigma_2(t)=\tanh(t)+\beta_{\mathrm{shift}}\) gives the model
\[
\mathcal T_{W,\theta}(H;X)
=\tanh(H\star W)+\beta_{\mathrm{shift}}\mathbf 1+B_\theta(X),
\]
which is positive and contractive in the Thompson metric for a sufficiently
large shift. For graphs, the corresponding baseline uses APPNP propagation~\citep{gasteiger2018predict},
so the normalized propagation operator remains fixed throughout the recurrence.
\rev{Appendix~\ref{app:subdeq} gives its exact form and contraction result.}

\paragraph{Sheaf convolution.} A standard graph propagation operator assigns a
scalar coefficient to each edge, shared across feature coordinates: it
controls \textit{how much} of a neighbour's representation is passed. Sheaf propagation
can instead transform that representation through an edge-specific matrix
before aggregation. We consider cellular sheaves with $d$-dimensional node and
edge stalks $\mathcal F(v)\cong\mathbb R^d$ and
$\mathcal F(e)\cong\mathbb R^d$. These stalks are connected by a linear
restriction map
$\mathcal F_{v\triangleleft e}:\mathcal F(v)\to\mathcal F(e)$ for every
incidence $v\triangleleft e$~\citep{curry2014sheaves,hansen2019toward}. For an
edge $e=\{u,v\}$, these restriction maps induce the block
$\mathcal F_{u\triangleleft e}^{\top}\mathcal F_{v\triangleleft e}$, which
maps the stalk at $v$ to the stalk at $u$ through their shared edge stalk.

To represent these stalk-valued features using the GNN convention, let $q$ be
the number of sheaf feature channels, so the total node embedding dimension is
$c=dq$. For $H\in\mathbb R^{n\times c}$, introduce the fixed reshape
$\mathscr R:\mathbb R^{n\times c}\to\mathbb R^{nd\times q}$. Write
$\mathbf H:=\mathscr R(H)$ and let
$\mathbf H_v\in\mathbb R^{d\times q}$ denote its block for node $v$.
Collecting the edgewise transformations over the graph gives the sheaf
\rev{Laplacian $L_{\mathcal F}$.} Let
$D_G:=\operatorname{diag}(\deg(v):v\in V)$ be the scalar node-degree matrix,
and assume that $G$ is connected. The scalar-degree-normalized sheaf
Laplacian and its propagation operator are
\[
\Delta_{\mathcal F}
    :=(D_G^{-1/2}\otimes I_d)L_{\mathcal F}
      (D_G^{-1/2}\otimes I_d),
    \qquad
\mathcal P_{\mathcal F}:=I_{nd}-\Delta_{\mathcal F}.
\]
Combining this propagation operator with stalk and channel mixing gives the
sheaf convolution
\[
H\star_{\mathcal F}[W_1,W_2]
:=\mathscr R^{-1}\!\left(
\mathcal P_{\mathcal F}(I_n\otimes W_1)\mathscr R(H)W_2
\right),
\]
where $W_1\in\mathbb R^{d\times d}$ mixes stalk coordinates and
$W_2\in\mathbb R^{q\times q}$ mixes feature channels. Unlike scalar graph
propagation, the induced edge matrix can align, mix, rescale, or reverse stalk
coordinates before aggregation. Since an observed graph generally does not provide restriction maps,
neural sheaf networks infer the two incidence maps of every undirected edge
from its endpoint representations~\citep{bodnar2022neural}.

\rev{Appendix~\ref{app:sheaves} details the block construction and the
properties of neural sheaf diffusion in explicit and implicit GNNs.}

\paragraph{Relation to prior work.}
Neural Sheaf Diffusion learns restriction maps from node representations but
applies them through a finite sequence of layers~\citep{bodnar2022neural}.
Deep Neural Sheaf Diffusion uses normalized, gated sheaf-adjacency propagation
at finite depth~\citep{bourgerie2026deep}. SheafDEQ
instead extends SubDEQ by replacing its fixed graph propagation with an
adaptive sheaf convolution whose restriction maps depend on the
\rev{recurrent state}~\citep{sittoni2024subhomogeneous}. EnergyGNN is the closest alternative
construction of a locally computable, asynchronously executable graph
equilibrium, but derives uniqueness from a strongly convex
PICNN-parameterized energy~\citep{amos2017input,solodova2025graph}; SheafDEQ
uses positivity, scale invariance, and subhomogeneity rather than a scalar
energy. Extended comparisons appear in Appendix~\ref{app:related-work}.

Table~\ref{tab:architecture-comparison} summarizes the resulting distinctions
among the evaluated architectures. GraphAttnDEQ and the Loop Transformer are
our \rev{finite-horizon attention ablation and control, respectively}, and are introduced in
Section~\ref{sec:experiments}.

\begin{table*}[t]
\centering
\small
\renewcommand{\arraystretch}{1.12}
\resizebox{\linewidth}{!}{%
\begin{tabular}{@{}lcccc@{}}
\toprule
Architecture & Propagation primitive & Well-posedness mechanism & Gradient computation & Inference procedure \\
\midrule
\multicolumn{5}{@{}l}{\textit{\rev{Finite-horizon recurrent models}}} \\
\cmidrule(r){1-5}
APPNP~\citep{gasteiger2018predict}
& graph convolution
& not required
& unrolled differentiation
& $K$-step recurrence \\

Loop Transformer (control)
& graph attention
& not required
& unrolled differentiation
& $K$-step recurrence \\

GraphAttnDEQ (ablation)
& graph attention
& not provided
& unrolled differentiation
& $K$-step recurrence \\
\addlinespace
\multicolumn{5}{@{}l}{\textit{Energy-defined equilibrium model}} \\
\cmidrule(r){1-5}
EnergyGNN~\citep{solodova2025graph}
& edge energy
& strongly convex energy
& implicit differentiation
& optimization \\
\addlinespace
\multicolumn{5}{@{}l}{\textit{Fixed-point equilibrium models}} \\
\cmidrule(r){1-5}
IGNN~\citep{gu2020implicit}
& graph convolution
& spectral-radius constraint
& implicit differentiation
& recurrence \\

EIGNN~\citep{liu2021eignn}
& graph convolution
& norm-bounded PSD mixing
& closed-form differentiation
& direct solution \\

APPNP-$\tanh$~\citep{sittoni2024subhomogeneous}
& graph convolution
& positive subhomogeneity
& unrolled differentiation
& recurrence \\

\textbf{SheafDEQ (ours)}
& \textbf{adaptive sheaf propagation}
& \textbf{scale invariance + subhomogeneity}
& \textbf{unrolled differentiation}
& \textbf{recurrence} \\
\bottomrule
\end{tabular}%
}
\caption{Comparison of the evaluated graph architectures.}
\label{tab:architecture-comparison}
\end{table*}

\section{SheafDEQ: adaptive neural-sheaf equilibria}
\label{sec:method}

Designing an implicit GNN with adaptive graph propagation is nontrivial: the
model must remain well posed even when the propagation operator depends on the
current \rev{state}. Existing implicit-GNN constructions typically obtain
well-posedness through restrictive parameterizations. We address this
challenge with SheafDEQ, which combines \rev{adaptive}, signed, and
matrix-valued sheaf propagation with a convergent equilibrium update.
Section~\ref{sec:adaptive-equilibrium} defines the equilibrium model;
Section~\ref{sec:orthant-routing} studies expressivity at nonlinear
equilibrium;
Section~\ref{sec:well-posedness} establishes uniqueness and synchronous
convergence; and Section~\ref{sec:asynchronous-robustness} extends the result
to bounded staleness.

\subsection{Adaptive sheaf equilibrium}
\label{sec:adaptive-equilibrium}
We now construct SheafDEQ with the objective of retaining both the
well-posedness of implicit GNNs and the expressive propagation of neural
sheaves. At iteration \(k\), \(H^{(k)}\) is the current \rev{recurrent state};
at convergence, \(H^\star\) is the \rev{equilibrium representation}. Since the observed graph does not
provide sheaf restriction maps, SheafDEQ infers them from the current
\rev{state} and recomputes them at each iteration. We first normalize this
\rev{state} before restriction-map inference.
For compactness, write \(\beta:=\beta_{\mathrm{shift}}\) throughout this
section.

\paragraph{State normalization.}
Recall that \(c=dq\) and that
\(\mathscr R:\mathbb R^{n\times c}\to\mathbb R^{nd\times q}\)
is the fixed reshape introduced in Section~\ref{sec:background}. For
\(H\in\mathbb R_{++}^{n\times c}\), write
\(\mathbf H:=\mathscr R(H)\), with
\(\mathbf H_v\in\mathbb R^{d\times q}\) denoting the block associated with
node \(v\). We define the pre-restriction-normalized \rev{state} by
\[
\overline{\mathbf H}:=\operatorname{norm}_p(\mathbf H),\qquad
[\overline{\mathbf H}]_{i,:}:=[\mathbf H]_{i,:}/\|[\mathbf H]_{i,:}\|_p,
\quad i=1,\ldots,nd.
\]
The normalization acts independently on each stalk-coordinate row across its
\(q\) feature channels. In particular, it is scale invariant:
\(
    \operatorname{norm}_p\!\bigl(\mathscr R(sH)\bigr)
    =
    \operatorname{norm}_p\!\bigl(\mathscr R(H)\bigr), s>0.
\)
\paragraph{Adaptive restriction maps.}
Although the graph is undirected, the model learns \rev{separate restriction
maps for its two incidences}. Each connection \(\{u,v\}\) is stored as \(u\to v\) and
\(v\to u\). For a stored edge \(e=(u\to v)\), set
\(z_e(H):=\operatorname{vec}(\overline{\mathbf H}_u)\Vert
\operatorname{vec}(\overline{\mathbf H}_v)\Vert a_e\) and define
\(
\mathcal F_{u\triangleleft e}(H):=\Pi(g_{\psi_{\mathrm{src}}}(z_e(H))),
\mathcal F_{v\triangleleft e}(H):=\Pi(g_{\psi_{\mathrm{tgt}}}(z_e(H))).
\)
For the reverse direction \(v\to u\), the same construction is applied with
the endpoint order reversed. Thus, the learned maps may differ between the two
directions of the same undirected connection.
Here \(g_{\psi_{\mathrm{src}}},g_{\psi_{\mathrm{tgt}}}:
\mathbb R^{2c+d_e}\to\mathbb R^{d\times d}\) are separate learnable MLPs
shared across all edges and \rev{iterations}.
The optional vector \(a_e\in\mathbb R^{d_e}\) contains edge features; when no
edge features are available, it and the final concatenation are omitted and
\(d_e=0\). The map \(\Pi(A):=A/\|A\|_{\mathrm F}\), defined for \(A\neq0\),
Frobenius-normalizes each restriction map. We assume the generators are
nonzero on admissible normalized inputs; this is included in the regularity
condition that \(\mathcal P_\Theta\) is well defined and Lipschitz below. We
write \(\psi=(\psi_{\mathrm{src}},\psi_{\mathrm{tgt}})\) and denote the resulting
state-dependent sheaf by \(\mathcal F_\psi(H;G)\).

\paragraph{Shared parameters and propagation.}
We collect the recurrent parameters as
\(
    \Theta:=\bigl(\theta,\psi,W_1,W_2\bigr).
\)
They are shared across \rev{iterations}, whereas
\(\beta\) is a fixed hyperparameter.

We replace the fixed graph convolution in SubDEQ with the corresponding
adaptive sheaf convolution
\[
\mathcal P_\Theta(H;G):=H\star_{\mathcal F_\psi(H;G)}[W_1,W_2]
:=\mathscr R^{-1}\!\left(\mathcal P_{\mathcal F_\psi(H;G)}
(I_n\otimes W_1)\mathscr R(H)W_2\right).
\]
Crucially, the normalized \rev{state} \(\overline{\mathbf H}\) is used only to
infer the restriction-map geometry; the convolution itself acts on the original
state \(H\).

\paragraph{Equilibrium update.}
Let \(\sigma_\beta(t):=\tanh(t)+\beta\). The resulting equilibrium update is
\[
\begin{aligned}
\mathcal T_\Theta(H;X,G)&:=\sigma_\beta(\mathcal P_\Theta(H;G))+B_\theta(X), \qquad
H^{\star}&=\mathcal T_\Theta(H^{\star};X,G).
\end{aligned}
\]
Here \(\sigma_\beta\) is applied entrywise,
\(B_\theta(X)\in\mathbb R^{n\times c}\) is entrywise nonnegative, and
\(\beta>1\). Since \(-1<\tanh(t)<1\), the complete update is
strictly positive even though the adaptive sheaf convolution may be signed and
matrix-valued.
Pre-restriction normalization makes the learned sheaf scale-invariant,
$\mathcal F_\psi(sH;G)=\mathcal F_\psi(H;G)$ for every $s>0$. Hence
$\mathcal P_\Theta(sH;G)=s\mathcal P_\Theta(H;G)$, even though the
restriction maps depend on $H$. This positive homogeneity is the key property
used in the well-posedness analysis below.
Starting from a positive \rev{initial state} \(H^{(0)}\), the equilibrium is
approximated using \(K\) Picard iterations, and a node-wise readout is applied
to \(H^{(K)}\). Parameters are trained by backpropagation through the unrolled
iterations. At test time, the recurrence can be continued beyond the training
\rev{iteration budget}.

\subsection{Expressivity at nonlinear equilibrium}
\label{sec:orthant-routing}

Existing expressivity results for neural sheaves concern linear
diffusion~\citep{bodnar2022neural}. SheafDEQ instead places signed,
matrix-valued sheaf propagation inside a positive shifted-\(\tanh\) fixed-point
update, so class separation need not automatically persist. We isolate this
question in the simplest fixed-sheaf case: one channel, no feature mixing, and
no input injection (\(q=1\), \(W_1=I_d\), \(W_2=1\), and
\(B_\theta(X)=0\)). Each class is assigned a sign vector
\(s_r\in\{-1,+1\}^d\): its \(+1\) entries identify coordinates that should
be larger, and its \(-1\) entries those that should be smaller. We construct
restriction maps whose equilibrium realizes these orderings.

\begin{proposition}[Signed separation at positive equilibrium]
\label{prop:orthant-routing}
Let \(G=(V,E)\) have no isolated nodes, let \(d\geq2\), and let
\(\ell:V\to\{1,\ldots,C\}\) be any labelling with
\(C\leq\binom{d}{\lfloor d/2\rfloor}\).
There exists a fixed \(d\)-dimensional cellular sheaf
\(\mathcal F_\ell\), with invertible unit-Frobenius restriction maps, such
that, under the standing assumption \(\beta>1\),
\(
    h^\star
    =
    \tanh(\mathcal P_{\mathcal F_\ell}h^\star)+\beta\mathbf1
\)
has a unique solution. More precisely, there are distinct sign vectors
\(s_1,\ldots,s_C\in\{-1,+1\}^d\), each with exactly
\(\lfloor d/2\rfloor\) negative entries, such that
\(h_{v,i}^\star>h_{v,j}^\star\) whenever
\(s_{\ell(v),i}=+1\) and \(s_{\ell(v),j}=-1\). Consequently, a linear
readout recovers \(\ell(v)\) at every node.
\end{proposition}

\rev{In this simple setting, Proposition~\ref{prop:orthant-routing} shows that
sheaf propagation can achieve class separation at a nonlinear equilibrium
despite the positive shift, nonlinearity, and equilibrium feedback. Since the
construction depends on the labels, this establishes what the model can
represent, not what training will learn. The proof is given in
Appendix~\ref{app:proof-orthant-routing}.
We next establish well-posedness of the full adaptive model.}

\subsection{Well-posedness}
\label{sec:well-posedness}

\paragraph{Invariant positive domain.}
Since \(-1<\tanh(t)<1\), when \(B_\theta(X)\geq 0\) entrywise and \(\beta>1\),
every positive state is mapped into the invariant set
\[
\mathcal K_X:=\{H:(\beta-1)\mathbf1\leq H\leq
(\beta+1)\mathbf1+B_\theta(X)\}\subset\mathbb R_{++}^{n\times c},
\]
where all inequalities are entrywise.
Thus, after one iteration, the states are bounded and uniformly separated
from zero. This invariant positive domain allows us to state the well-posedness
result.

\begin{theorem}[Well-posedness of SheafDEQ]
\label{thm:sheafdeq}
Fix \((X,G)\). Assume that
\(\mathcal P_\Theta(\cdot;G)\) is Lipschitz on
\(\mathbb R_{++}^{n\times c}\), that \(B_\theta(X)\geq 0\) entrywise, and that
\(\beta\geq1.2\). Then
\(\mathcal T_\Theta(\cdot;X,G)\) admits a
unique fixed point \(H^\star\in\mathcal K_X\). Moreover, for every
entrywise-positive \(H^{(0)}\), the iterates
\(H^{(k+1)}=\mathcal T_\Theta(H^{(k)};X,G)\) converge globally
and linearly to \(H^\star\), with some contraction factor \(\mu<1\). In
particular, for \(k\geq 1\),
\[
    \delta_{\mathrm T}(H^{(k)},H^\star)
    \leq
    \mu^{k-1}\delta_{\mathrm T}(H^{(1)},H^\star),
    \qquad
    \left\|H^{(k)}-H^\star\right\|_\infty
    \leq
    C_{\mathrm{fp}}\mu^{k-1},
\]
for some constant \(C_{\mathrm{fp}}>0\) depending on \(H^{(1)}\) and \(H^\star\).

\end{theorem}
\rev{The proof is given in Appendix~\ref{app:proof-3.2}.
Appendix~\ref{app:equilibrium-interpretation} discusses the sheaf
interpretation of the SheafDEQ equilibrium.}

The shift controls the strength of the contraction guarantee: a larger
\(\beta\) gives a smaller contraction factor \(\mu\), and
therefore a faster guaranteed linear convergence rate. For example,
\citet{sittoni2024subhomogeneous} report \(\mu=0.99\) for
\(\tanh(t)+1.2\), but \(\mu=0.499\) for \(\tanh(t)+1.603\). Thus,
\(\beta=1.2\) already ensures contractivity, while larger
shifts can provide a stronger guarantee. These bounds are worst-case
theoretical rates and need not match the observed convergence speed.

\subsection{Robustness to bounded staleness}
\label{sec:asynchronous-robustness}

\paragraph{Partial-asynchronism model.}
A useful consequence of contractivity is robustness to asynchronous
inference. Consider a time-slotted execution indexed by \(t\in\mathbb N\).
Following the partial-asynchronism model of \citet{solodova2025graph}, let
\(\mathcal I_i\subseteq\mathbb N\) denote the update times of node \(i\), and
let \(\tau_j^i(t)\leq t\) denote the \rev{timestamp of node \(j\)'s state} used by node
\(i\). It is assumed that, for every node \(i\) and neighbour \(j\),

\[
\mathcal I_i\cap\{t,\ldots,t+s_{\mathrm{upd}}-1\}\neq\varnothing,
\qquad
t-\tau_{\max}\leq\tau_j^i(t)\leq t,
\qquad
\tau_i^i(t)=t.
\]

Thus, every node is updated at least once every \(s_{\mathrm{upd}}\) slots,
neighbouring states are delayed by at most \(\tau_{\max}\) slots, and each node
uses its current local state.

At an update time, node \(i\) evaluates the same block map
\(\mathcal T_{\Theta,i}\), but using the possibly stale view
\(\widetilde H^i(t)\), where
\([\widetilde H^i(t)]_j=H_j(\tau_j^i(t))\):
\[
H_i(t+1)
=
\begin{cases}
\mathcal T_{\Theta,i}
    \bigl(\widetilde H^i(t);X,G\bigr),
    & t\in\mathcal I_i,\\
H_i(t),
    & t\notin\mathcal I_i.
\end{cases}
\]

\begin{corollary}[Convergence under bounded staleness]
\label{cor:asynchronous}
Under the assumptions of Theorem~\ref{thm:sheafdeq}, the partially
asynchronous iterates defined above converge from every entrywise-positive
\rev{initial state} to the unique equilibrium \(H^\star\).
\end{corollary}

This result is the standard asynchronous robustness property of contractive
fixed-point iterations. A proof in our setting is given in
Appendix~\ref{app:asynchronous-proof}.

\section{Experiments}
\label{sec:experiments}
\paragraph{\rev{Comparison models and training protocol.}}
\rev{We evaluate SheafDEQ on synthetic benchmarks along three dimensions:
predictive performance, continued-iteration behaviour, and sensitivity to
bounded communication staleness. Our external baselines comprise the
implicit models introduced in Section~\ref{sec:background} and finite-depth
graph and neural-sheaf models. We additionally evaluate three models: two
ablations and a finite-horizon control. SheafDEQ without pre-restriction
normalization is the equilibrium ablation: it removes only this normalization
and does not inherit the default model's convergence guarantee. GraphAttnDEQ
is the sheaf-propagation ablation: starting from the no-normalization variant,
it replaces learned sheaf propagation with neighbourhood softmax attention.
Both ablations retain the recurrent skeleton, fixed input drive, and additive
shift. The Loop Transformer is a finite-horizon control that repeats a
graph-masked Transformer block, including residual connections, LayerNorm,
and a feedforward sublayer, for $K$ iterations. GraphAttnDEQ and the Loop
Transformer recompute attention from the current representations at each
iteration and are trained through the finite unroll; neither has an
equilibrium or convergence guarantee. All models use Adam and
validation-based early stopping. Complete dataset, architecture, training,
and model-selection details are provided in
Appendices~\ref{app:experiments} and~\ref{app:baselines}.}

\subsection{\rev{Predictive performance}}
\label{sec:predictive-benchmarks}

\paragraph{Noisy community detection.}
We use the benchmark of \citet{bourgerie2026deep}, building on
\citet{zaghen2024sheaf}. Each graph contains 1,500 nodes in three balanced
communities with overlapping two-dimensional observations. At rewiring level
\(L_{\mathrm{hetero}}\), \(10L_{\mathrm{hetero}}\%\) of the edges are replaced
by cross-community edges. Models are trained on six independently generated
graphs and evaluated on three disjoint graphs. Low rewiring provides
homophilic evidence, whereas high rewiring creates systematic cross-community
structure, testing whether learned edge transformations can exploit both
regimes. DNSD provides an explicit neural-sheaf comparison, GraphAttnDEQ
isolates the propagation mechanism, and the Loop Transformer is an adaptive
finite-horizon control. Full dataset details appear in
Appendix~\ref{app:datasets}.

\begin{table*}[t]
\centering
\caption{\textbf{Community detection.} Test accuracy (\%) across rewiring
levels \(L_{\mathrm{hetero}}\in\{1,3,5,7,9,10\}\), reported as
median \([\min,\max]\) over five parameter-initialization seeds. For each model
family, we select the configuration with the highest mean validation accuracy
at each level; the two \textsc{SheafDEQ} normalization variants are selected
separately. Pre-norm \(=\checkmark\) denotes the proposed model and pre-norm
\(=\times\) its ablation. \rone{1st}, \rtwo{2nd}, and \rthree{3rd} mark the
highest test medians. Full results appear in Appendix
Table~\ref{tab:community-full}.}
\label{tab:community}
\resizebox{\textwidth}{!}{%
\begin{tabular}{@{}l l c c c c c c@{}}
\toprule
\multicolumn{2}{l}{\textbf{Model}} & \multicolumn{6}{c}{\textbf{$L_{\mathrm{hetero}}$}} \\
\cmidrule(l){3-8}
\multicolumn{2}{l}{} & \textbf{1} & \textbf{3} & \textbf{5} & \textbf{7} & \textbf{9} & \textbf{10} \\
\midrule
\multicolumn{8}{@{}l}{\textit{\rev{Non-adaptive models}}} \\
\cmidrule(r){1-8}
  \multicolumn{2}{l}{MLP} & 33.6\,\scriptsize{[33.3,34.8]} & 33.3\,\scriptsize{[33.3,37.4]} & 33.3\,\scriptsize{[32.8,33.4]} & 33.3\,\scriptsize{[31.5,33.5]} & 33.3\,\scriptsize{[33.3,34.0]} & 33.3\,\scriptsize{[33.3,36.3]} \\
  \multicolumn{2}{l}{APPNP} & 40.0\,\scriptsize{[39.4,40.5]} & 40.8\,\scriptsize{[39.4,40.8]} & 40.7\,\scriptsize{[40.6,41.0]} & 40.3\,\scriptsize{[40.0,40.8]} & 40.4\,\scriptsize{[40.3,40.7]} & 40.0\,\scriptsize{[39.5,40.2]} \\
  \multicolumn{2}{l}{IGNN} & \rthree{42.2\,\scriptsize{[40.9,42.8]}} & 43.6\,\scriptsize{[43.0,44.6]} & 44.6\,\scriptsize{[43.9,45.0]} & 45.0\,\scriptsize{[44.8,45.4]} & 45.8\,\scriptsize{[45.5,46.7]} & 52.7\,\scriptsize{[51.9,54.2]} \\
\midrule
\multicolumn{8}{@{}l}{\textit{\rev{Adaptive finite-depth/horizon models}}} \\
\cmidrule(r){1-8}
  \multicolumn{2}{l}{DNSD} & \rone{71.4\,\scriptsize{[42.1,80.3]}} & \rone{87.5\,\scriptsize{[57.1,93.9]}} & \rthree{90.4\,\scriptsize{[84.0,93.2]}} & 89.6\,\scriptsize{[64.7,92.7]} & 63.5\,\scriptsize{[50.9,72.7]} & 97.4\,\scriptsize{[94.0,98.0]} \\
  \multicolumn{2}{l}{Loop Transformer \rev{(control)}} & 40.4\,\scriptsize{[40.0,43.4]} & 41.0\,\scriptsize{[38.7,42.9]} & 41.8\,\scriptsize{[40.7,45.5]} & \rthree{98.8\,\scriptsize{[97.0,99.2]}} & \rthree{90.5\,\scriptsize{[86.2,92.6]}} & \rone{99.1\,\scriptsize{[98.8,99.1]}} \\
  \multicolumn{2}{l}{GraphAttnDEQ \rev{(ablation)}} & \rtwo{42.6\,\scriptsize{[42.0,44.2]}} & 46.8\,\scriptsize{[43.5,59.9]} & 42.6\,\scriptsize{[40.9,45.2]} & 41.3\,\scriptsize{[39.0,59.3]} & 72.3\,\scriptsize{[54.3,81.9]} & 72.0\,\scriptsize{[68.9,98.3]} \\
\midrule
\multicolumn{8}{@{}l}{\textit{\rev{Adaptive equilibrium model}}} \\
\cmidrule(r){1-8}
  \multirow{2}{*}{\textsc{SheafDEQ}} & pre-norm=\!\checkmark & 42.0\,\scriptsize{[40.8,44.3]} & \rthree{59.8\,\scriptsize{[45.7,96.5]}} & \rtwo{91.6\,\scriptsize{[51.3,99.4]}} & \rone{99.0\,\scriptsize{[94.8,99.1]}} & \rone{98.0\,\scriptsize{[95.6,98.3]}} & \rthree{98.6\,\scriptsize{[96.9,99.1]}} \\
   & pre-norm=\!$\times$ \rev{(ablation)} & 41.1\,\scriptsize{[40.5,63.4]} & \rtwo{61.7\,\scriptsize{[44.1,98.0]}} & \rone{98.3\,\scriptsize{[77.9,99.7]}} & \rone{99.0\,\scriptsize{[73.1,99.4]}} & \rone{98.0\,\scriptsize{[97.5,99.0]}} & \rtwo{98.8\,\scriptsize{[98.6,99.1]}} \\
\bottomrule
\end{tabular}%
}
\end{table*}

\paragraph{Community-detection results.}
\rev{Table~\ref{tab:community} reports test accuracy across rewiring levels.
SheafDEQ's advantage over the attention ablation is clearest when within-
and cross-community edges coexist. This suggests a benefit from transforming
each neighbour's representation differently before aggregation. Under complete
rewiring, all neighbours belong to other communities, which may make
cross-community information easier to use; the finite-horizon control also
performs strongly in this setting. At low rewiring, DNSD outperforms the
recurrent models. With fewer cross-community shortcuts, information must
travel farther through local connections. DNSD can learn different
transformations across layers rather than repeat the same recurrent update,
which may better support this longer-range propagation. Full configurations
and training-variability diagnostics appear in
Appendix~\ref{app:community-notes}.}

\paragraph{Synthetic distributed tasks.}
We use four distributed-inference benchmarks from
\citet{solodova2025graph}: Counting tests structural counting, Sums global aggregation, MNIST Terrain
collective classification, and Coordinates relational localization from
pairwise distances. The first three tasks use transductive node-level splits,
whereas Coordinates is evaluated on held-out graphs. Full dataset
constructions and split protocols appear in
Appendix~\ref{app:datasets} and Appendix~\ref{app:splits}.

\begin{table*}[t]
\centering
\caption{\rev{\textbf{Synthetic distributed tasks.} Test performance reported as
median \([\min,\max]\) over five parameter-initialization seeds; higher is
better. For each model family, we select the configuration with the highest
mean validation performance for each task; the two \textsc{SheafDEQ}
normalization variants are selected separately. The \(\dagger\) denotes that
the EnergyGNN optimization solver did not converge for any seed; no predictive
score is reported. \rone{1st}, \rtwo{2nd}, and
\rthree{3rd} mark the highest medians; ties at the displayed precision share
a rank. Full results and baseline-coverage
details appear in Appendix Table~\ref{tab:solodova-full} and
Appendix~\ref{app:benchmark-notes}.}}
\label{tab:benchmark}
\resizebox{\textwidth}{!}{%
\begin{tabular}{@{}l l c c c c@{}}
\toprule
\multicolumn{2}{l}{\textbf{Model}} & \textbf{Counting} & \textbf{Sums} & \textbf{MNIST} & \textbf{Coordinates} \\
\multicolumn{2}{l}{} & \scriptsize{$-\mathrm{MSE}\ \uparrow$} & \scriptsize{$-\mathrm{MSE}\ \uparrow$} & \scriptsize{\textit{Acc. $\uparrow$}} & \scriptsize{$-\mathcal{L}_{\mathrm{dist}}\ \uparrow$} \\
\midrule
\multicolumn{6}{@{}l}{\textit{\rev{Non-adaptive and energy-based models}}} \\
\cmidrule(r){1-6}
  \multicolumn{2}{l}{MLP} & -0.367\,\scriptsize{[-0.574,-0.347]} & -0.154\,\scriptsize{[-0.999,-0.145]} & 53.3\,\scriptsize{[53.3,53.3]} & -0.105\,\scriptsize{[-0.106,-0.100]} \\
  \multicolumn{2}{l}{APPNP} & -0.335\,\scriptsize{[-0.348,-0.309]} & -0.136\,\scriptsize{[-0.140,-0.134]} & 74.8\,\scriptsize{[74.3,75.3]} & -0.030\,\scriptsize{[-0.032,-0.029]} \\
  \multicolumn{2}{l}{IGNN} & -0.336\,\scriptsize{[-0.681,-0.307]} & \rev{\rthree{-0.130\,\scriptsize{[-0.998,-0.128]}}} & 76.1\,\scriptsize{[75.6,79.9]} & -0.028\,\scriptsize{[-0.028,-0.027]} \\
  \multicolumn{2}{l}{EnergyGNN} & -0.705\,\scriptsize{[-0.895,-0.305]} & -0.184\,\scriptsize{[-0.589,-0.140]} & 60.4\,\scriptsize{[53.3,68.2]} & \rev{---\({}^{\dagger}\)} \\
\midrule
\multicolumn{6}{@{}l}{\textit{\rev{Adaptive finite-depth/horizon models}}} \\
\cmidrule(r){1-6}
  \multicolumn{2}{l}{Loop Transformer \rev{(control)}} & -0.335\,\scriptsize{[-0.361,-0.280]} & -0.140\,\scriptsize{[-0.143,-0.139]} & \rev{\rone{95.7\,\scriptsize{[95.3,95.9]}}} & \rone{-0.006\,\scriptsize{[-0.009,-0.006]}} \\
  \multicolumn{2}{l}{GraphAttnDEQ \rev{(ablation)}} & \rev{\rone{-0.316\,\scriptsize{[-0.341,-0.283]}}} & \rev{\rtwo{-0.128\,\scriptsize{[-0.138,-0.120]}}} & \rev{\rtwo{91.4\,\scriptsize{[90.9,92.3]}}} & -0.028\,\scriptsize{[-0.029,-0.012]} \\
\midrule
\multicolumn{6}{@{}l}{\textit{\rev{Adaptive equilibrium model}}} \\
\cmidrule(r){1-6}
  \multirow{2}{*}{\textsc{SheafDEQ}} & pre-norm=\!\checkmark & \rev{\rtwo{-0.317\,\scriptsize{[-0.359,-0.304]}}} & \rev{\rone{-0.127\,\scriptsize{[-0.139,-0.124]}}} & 89.9\,\scriptsize{[87.8,90.6]} & \rtwo{-0.009\,\scriptsize{[-0.014,-0.007]}} \\
   & pre-norm=\!$\times$ \rev{(ablation)} & \rev{\rthree{-0.326\,\scriptsize{[-0.351,-0.312]}}} & \rev{\rtwo{-0.128\,\scriptsize{[-0.134,-0.125]}}} & \rev{\rthree{91.1\,\scriptsize{[87.4,93.4]}}} & \rthree{-0.014\,\scriptsize{[-0.019,-0.011]}} \\
\bottomrule
\end{tabular}%
}
\end{table*}

\paragraph{Distributed-task results.}
\rev{Table~\ref{tab:benchmark} reports test performance across the four
distributed tasks. SheafDEQ and the attention ablation perform similarly on
Counting, Sums, and MNIST Terrain, whereas SheafDEQ performs better on
Coordinates. We interpret this advantage as a benefit of edge-dependent
transformations when combining information from different pairwise distance
constraints. Counting and Sums instead require repeated aggregation along
chains, where attention and sheaf propagation give similar results. On MNIST
Terrain, both mechanisms adapt propagation to the current representations
when combining pixel observations, and both outperform fixed-propagation
baselines. The finite-horizon control outperforms SheafDEQ on MNIST Terrain
and Coordinates, suggesting that its richer recurrent block also benefits
these tasks. Further discussion appears in
Appendix~\ref{app:benchmark-notes}.}

\begin{tcolorbox}[colback=black!3,colframe=black!30,boxrule=0.5pt,arc=2pt]
\rev{\textbf{Summary of predictive results.} Across both experiment sets, the results support adaptive propagation over
fixed propagation. The additional
benefit of sheaf propagation is clearest when neighbours require different
treatment (e.g., mixed heterophily, geometry). SheafDEQ remains competitive with recurrent
models without equilibrium guarantees.}
\end{tcolorbox}

\rev{While these experiments show no consistent predictive advantage from the
equilibrium-guarantee mechanism at the trained iteration budget, we next
examine its role when computation continues beyond that budget or uses
delayed neighbour information.}

\subsection{Empirical test-time \rev{recurrence} diagnostics}
\label{sec:equilibrium-robustness}

\paragraph{Diagnostic protocol.}
\rev{We evaluate whether the test-time recurrence settles across
different initial states and remains predictive beyond its training horizon.
For each SheafDEQ variant, across six positive initial-state conditions, we
report the median relative one-step update residual \(\rho_k\),
together with test accuracy as recurrence continues through 100 iterations.
To assess dependence on the initial state directly, we compare each
trajectory's state and prediction at iteration \(k\) with those of the default
trajectory after 100 iterations. The construction of the initial states, the evaluation metrics, and
the full diagnostic definitions are provided in
Appendix~\ref{app:equilibrium-protocol}.}

\begin{figure*}[t]
    \centering
    \includegraphics[width=\textwidth]{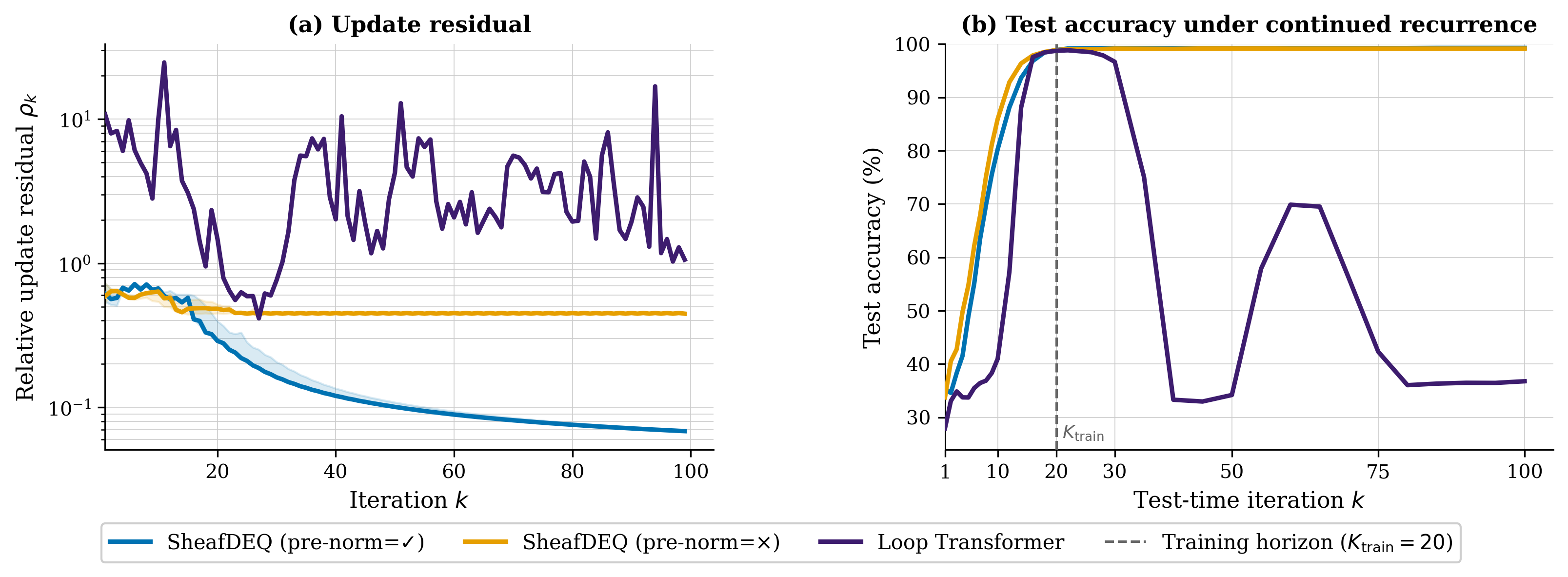}
\caption{\textbf{Finite-iteration recurrence diagnostics.}
Panels show (a) the median relative update residual
\rev{and (b) test accuracy under continued recurrence}.
For the Loop Transformer, panel~(a) uses one-step logit updates.
\rev{For SheafDEQ,} solid curves and shading \rev{in panel~(a)} show medians and
interquartile ranges across six \rev{initial-state conditions}. \rev{The Loop
Transformer does not accept an externally specified recurrent initial state,
so its curve is unshaded;} the dashed line
\rev{in panel~(b)} marks its training horizon \(K_{\mathrm{train}}=20\).
}

\label{fig:equilibrium-robustness}
\end{figure*}

\paragraph{\rev{Diagnostic results}.}
\rev{Figure~\ref{fig:equilibrium-robustness} reports relative update residuals and test accuracy under continued recurrence for SheafDEQ, its equilibrium-guarantee ablation, and the finite-horizon control. In panel~(a), SheafDEQ shows decreasing residuals, consistent with the convergence predicted by Theorem~\ref{thm:sheafdeq}, whereas its ablation retains larger residuals. Panel~(b) shows that both retain their test accuracy through 100 iterations, while the finite-horizon control eventually loses accuracy beyond its trained 20-loop horizon.}
\rev{Interestingly, the ablation retains similar accuracy despite its larger residuals. We hypothesize that the shared saturating nonlinearity limits the effect of continued state changes on classification. This suggests that retaining accuracy does not require the state to have settled. Additional state- and prediction-discrepancy diagnostics examine dependence on the initial state in Appendix~\ref{app:equilibrium-protocol}, Figure~\ref{fig:equilibrium-robustness-full}.}

\subsection{Sensitivity to bounded communication staleness}
\label{sec:staleness-results}

\paragraph{Delayed-update protocol.}
\rev{To measure sensitivity to bounded communication staleness,} we evaluate
one trained checkpoint for each \rev{model} on the fixed rewiring-level-$7$
community-detection graph under sampled partially asynchronous schedules,
following the protocol of \citet{solodova2025graph}.
Each schedule updates every node within \(s_{\mathrm{upd}}\) slots and uses
neighbour information at most \(\tau_{\max}\) slots old. Starting from the
same \rev{initial state}, we sample 20 schedules for
\((s_{\mathrm{upd}},\tau_{\max})\in
\{(1,0),(2,1),(5,2),(8,4)\}\), where \((1,0)\) is synchronous, and report the
relative Frobenius discrepancy \(d_{\mathrm{out}}\) from the corresponding
synchronous prediction. The SheafDEQ variants---with and without
pre-restriction normalization---are evaluated after \(u=500\)
mean local updates per node, using synchronous references at 420 iterations
for the former and 500 iterations for the latter; the Loop
Transformer uses its synchronous 20-loop output.
Complete definitions, protocol, and results appear in
Appendix~\ref{app:staleness-experiment}.

\begin{table}[h!]
\centering
\small
\caption{\rev{\textbf{Bounded-staleness sensitivity at
\((s_{\mathrm{upd}},\tau_{\max})=(5,2)\).}
Median \([\min,\max]\) relative logit discrepancy \(d_{\mathrm{out}}\) across 20 schedules;
lower is better. Complete results appear in
Appendix~\ref{app:staleness-experiment}.}}
\label{tab:async-main}
{
\begin{tabular}{@{}lc@{}}
\toprule
Model & \(d_{\mathrm{out}}\) \\
\midrule
SheafDEQ, pre-norm \(=\checkmark\)
& \(7.08{\times}10^{-4}\ [7.08,7.09]{\times}10^{-4}\) \\
SheafDEQ, pre-norm \(=\times\)
& \(4.12{\times}10^{-2}\ [1.57,5.32]{\times}10^{-2}\) \\
Loop Transformer (20 loops)
& \(7.09{\times}10^{-2}\ [6.32,7.93]{\times}10^{-2}\) \\
\bottomrule
\end{tabular}}
\end{table}

\paragraph{Staleness results.}
\rev{Table~\ref{tab:async-main} reports relative logit discrepancies from synchronous inference under the intermediate delayed condition, with \(s_{\mathrm{upd}}=5\) and \(\tau_{\max}=2\). SheafDEQ remains close to its synchronous prediction, whereas its equilibrium-guarantee ablation shows larger discrepancies and greater variation across schedules. This separation persists across all tested delay settings (Appendix~\ref{app:staleness-experiment}, Table~\ref{tab:async-full}, Figure~\ref{fig:async-trajectories}).}
\rev{Unlike the continued-recurrence experiment, where both variants retain similar accuracy, delayed communication exposes a clearer benefit of the equilibrium guarantee: SheafDEQ's output depends little on the sampled update schedule. This is consistent with Corollary~\ref{cor:asynchronous}, which guarantees convergence to the same equilibrium under bounded staleness.}

\section{Conclusion}

\rev{We introduced SheafDEQ, an implicit GNN with adaptive, matrix-valued sheaf
propagation. Under explicit regularity conditions, fixed-point iteration
converges globally to a unique equilibrium; contractivity also guarantees
convergence under bounded communication staleness.}

\rev{Across both experiment sets, the results support adaptive propagation,
with sheaf propagation most beneficial when neighbours require different
treatment. SheafDEQ remains competitive with recurrent models without
equilibrium guarantees. Continued recurrence preserves predictive accuracy
even before the state settles, while delayed-update experiments show low
sensitivity to communication staleness, consistent with the asynchronous
convergence guarantee.}

\rev{Variation across parameter-initialization seeds remains the main training limitation
\citep{agarwala2022deep} and may prevent SheafDEQ from consistently realizing
the expressive potential of adaptive sheaf propagation. Future work should
improve training reliability and evaluate transfer to inductive, real-world
graph datasets.}

\section*{Acknowledgement}
The computations were enabled by resources provided by the National Academic Infrastructure for Supercomputing in Sweden (NAISS), partially funded by the Swedish Research Council through grant agreement no. 2022-06725.

\bibliographystyle{unsrtnat}
\bibliography{references}

\appendix
\section{Notation}
\label{app:notation}

Tables~\ref{tab:notation-core} and~\ref{tab:notation-experiments} collect the
recurring notation used throughout the paper. Local proof variables are
defined where they appear. Matrix inequalities involving \rev{recurrent states} are
entrywise, and $\|H\|_\infty=\max_{i,j}|H_{ij}|$.

\begin{table*}[t]
\centering
\small
\caption{Core graph, sheaf, and equilibrium notation.}
\label{tab:notation-core}
\renewcommand{\arraystretch}{1.08}
\begin{tabularx}{\textwidth}{@{}l l Y@{}}
\toprule
Symbol & Shape/domain & Meaning \\
\midrule
$G=(V,E)$ & $|V|=n$ & Undirected graph with node set $V$ and edge set $E$. \\
$A$ & $\mathbb R^{n\times n}$ & Raw adjacency matrix, without added self-loops. \\
$D_G$ & $\mathbb R^{n\times n}$ & Scalar node-degree matrix $D_G=\operatorname{diag}(A\mathbf1)$. \\
$A^+,D^+$ & $\mathbb R^{n\times n}$ & Self-loop quantities $A^+=A+I$ and $D^+=\operatorname{diag}(A^+\mathbf1)$. \\
$\widehat A$ & $\mathbb R^{n\times n}$ & Symmetric normalized adjacency $D_G^{-1/2}AD_G^{-1/2}$. \\
$\widehat A^+$ & $\mathbb R^{n\times n}$ & Symmetric normalized adjacency with self-loops $(D^+)^{-1/2}A^+(D^+)^{-1/2}$. \\
$X$ & $\mathbb R^{n\times f}$ & Raw node-feature matrix; $f$ is the node-feature dimension. \\
$d,q,c$ & $c=dq$ & Stalk dimension, sheaf feature-channel count, and total node \rev{representation} dimension. \\
$C$ & $\mathbb N$ & Number of prediction classes in a classification task. \\
$o$ & $\mathbb N$ & Per-node output dimension: $o=C$ for classification, $o=1$ for scalar regression, and $o=2$ for Coordinates. \\
$H$ & $\mathbb R^{n\times c}$ & Recurrent state in the common row-major GNN convention. \\
$\mathscr R(H)=\mathbf H$ & $\mathbb R^{nd\times q}$ & Fixed reshape of $H$ into the sheaf-stacked \rev{state}; $\mathbf H_v\in\mathbb R^{d\times q}$ is node $v$'s block. \\
$B_\theta(X)$ & $\mathbb R^{n\times c}$ & Entrywise-nonnegative input injection. \\
$\overline{\mathbf H}$ & $\mathbb R^{nd\times q}$ & Pre-restriction-normalized sheaf state $\operatorname{norm}_p(\mathbf H)$, normalized independently along each row. \\
$p$ & $[1,\infty]$ & Norm order used by the pre-restriction row normalization. \\
$\mathcal F$ & cellular sheaf on $G$ & Collection of node and edge stalks and their incidence maps. \\
$\mathcal F_{v\triangleleft e}$ & $\mathbb R^{d\times d}$ & Restriction map from the stalk at $v$ to the stalk at incident edge $e$. \\
$g_{\psi_{\mathrm{src}}},g_{\psi_{\mathrm{tgt}}}$ & $\mathbb R^{2c+d_e}\to\mathbb R^{d\times d}$ & Separately parameterized incidence-map generators receiving the same ordered endpoint concatenation. \\
$z_e(H)$ & $\mathbb R^{2c+d_e}$ & Ordered endpoint \rev{states} and optional edge features used by both incidence-map generators. \\
$a_e$ & $\mathbb R^{d_e}$ & Optional feature vector of edge $e$. \\
$\Pi$ & $\mathbb R^{d\times d}\setminus\{0\}\to\mathbb R^{d\times d}$ & Frobenius normalization applied to a nonzero generated restriction map. \\
$W_1,W_2$ & $\mathbb R^{d\times d}$, $\mathbb R^{q\times q}$ & Stalk-coordinate and sheaf-channel mixing matrices. \\
$\operatorname{norm}_p$ & row-wise map & Per-row $p$-norm normalization used only before restriction inference. \\
$\delta_{\mathcal F}$ & linear map & Sheaf coboundary operator associated with an arbitrary orientation of the undirected edges. \\
$L_{\mathcal F}$ & $\mathbb R^{nd\times nd}$ & Unnormalized sheaf Laplacian $\delta_{\mathcal F}^\top\delta_{\mathcal F}$. \\
$\Delta_{\mathcal F}$ & $\mathbb R^{nd\times nd}$ & Scalar-degree-normalized sheaf Laplacian $(D_G^{-1/2}\otimes I_d)L_{\mathcal F}(D_G^{-1/2}\otimes I_d)$. \\
$\mathcal P_{\mathcal F}$ & $\mathbb R^{nd\times nd}$ & Fixed-sheaf propagation operator $I_{nd}-\Delta_{\mathcal F}$. \\
$\mathcal P_\Theta(H;G)$ & $\mathbb R^{n\times c}$ & Adaptive sheaf propagation induced by $\mathcal F_\psi(H;G)$, reshaped back to the common \rev{recurrent-state} convention. \\
$\sigma_{\beta_{\mathrm{shift}}}$ & $\mathbb R\to\mathbb R_{++}$ & Shifted activation $t\mapsto\tanh(t)+\beta_{\mathrm{shift}}$. \\
$\beta_{\mathrm{shift}}$ & $(1,\infty)$ & Positive scalar shift in the SheafDEQ activation. \\
$\beta_{\mathrm E}$ & $\mathbb R_{++}$ & Strong-convexity coefficient in the EnergyGNN baseline. \\
$\mathcal T_\Theta(H;X,G)$ & $\mathbb R_{++}^{n\times c}\to\mathbb R_{++}^{n\times c}$ & Complete SheafDEQ equilibrium update. \\
$H^\star$ & $\mathbb R_{++}^{n\times c}$ & Unique equilibrium \rev{representation} satisfying $H^\star=\mathcal T_\Theta(H^\star;X,G)$. \\
$\mathcal K_X$ & subset of $\mathbb R_{++}^{n\times c}$ & Positive invariant order interval used in the fixed-point theorem. \\
$\Theta$ & parameter collection & All parameters of the recurrent SheafDEQ model; $\theta$ and $\psi$ denote specified subsets. \\
$r_\omega$ & $\mathbb R^c\to\mathbb R^o$ & Readout, where $o=C$ for classification, $o=1$ for scalar regression, and $o=2$ for Coordinates. \\
$\widehat Y^{(k)}$ & $\mathbb R^{n\times o}$ & Prediction or logit output obtained by applying $r_\omega$ row-wise to $H^{(k)}$. \\
$\mu$ & $(0,1)$ & Subhomogeneity and Thompson-contraction coefficient. \\
$C_{\mathrm{fp}}$ & $\mathbb R_{++}$ & Constant in the entrywise fixed-point convergence bound. \\
$\delta_{\mathrm T}$ & metric & Thompson metric $\delta_{\mathrm T}(H,\widehat H)=\|\log H-\log\widehat H\|_\infty$. \\
\bottomrule
\end{tabularx}
\end{table*}

\begin{table*}[t]
\centering
\small
\caption{Iteration, benchmark, and diagnostic notation.}
\label{tab:notation-experiments}
\renewcommand{\arraystretch}{1.08}
\begin{tabularx}{\textwidth}{@{}l l Y@{}}
\toprule
Symbol & Domain & Meaning \\
\midrule
$k$ & $\mathbb N$ & Iteration index. \\
$K$ & $\mathbb N$ & Finite iteration or layer budget. \\
$K_{\mathrm{train}}$ & $\mathbb N$ & Iteration horizon used during training. \\
$K_{\mathrm{ref}}$ & $\mathbb N$ & \rev{Iteration used to define the numerical reference state and prediction}. \\
$k_{\mathrm{NN}}$ & $\mathbb N$ & Number of nearest neighbours used to construct a nearest-neighbour graph. \\
$L_{\mathrm{hetero}}$ & $\{0,\ldots,10\}$ & Community-detection perturbation level. \\
$\sigma_{\mathrm{obs}}$ & $\mathbb R_{++}$ & Standard deviation of the synthetic community node observations. \\
$j$ & $\{1,\ldots,6\}$ & \rev{Initial-state condition} index in the equilibrium diagnostic. \\
$\rho_{j,k}$ & $\mathbb R_+$ & Relative \rev{state} update residual for \rev{initial-state condition} $j$. \\
$\rho_k$ & $\mathbb R_+$ & Median of $\rho_{j,k}$ across \rev{initial-state conditions}. \\
$d_k^{\mathrm{state}}$ & $\mathbb R_+$ & Median relative \rev{state} discrepancy across \rev{initial-state conditions}. \\
$d_k^{\mathrm{pred}}$ & $\mathbb R_+$ & Relative prediction discrepancy from the \rev{numerical reference prediction}. \\
$\varepsilon$ & $\mathbb R_{++}$ & Small numerical stabilizer in relative diagnostic denominators. \\
$p_{\mathrm{drop}}$ & $[0,1)$ & Dropout probability. \\
$\eta_{\mathrm{lr}}$ & $\mathbb R_{++}$ & Optimizer learning rate. \\
$\lambda_{\mathrm{wd}}$ & $\mathbb R_+$ & Weight-decay coefficient. \\
$W_{\mathrm{rec}}$ & $\mathbb R^{c\times c}$ & Recurrent feature-mixing matrix of the IGNN baseline. \\
$\widehat\rho_{\mathrm{PI}}$ & $\mathbb R^{n\times n}\to\mathbb R_+$ & Spectral-radius estimator obtained by power iteration in the IGNN projection. \\
$s_{\mathrm{upd}}$ & $\mathbb N$ & Maximum update-spacing window in the sampled asynchronous schedules. \\
$\tau_{\max}$ & $\mathbb N_0$ & Maximum communication delay in logical slots. \\
$B_{\mathrm{async}}$ & $\mathbb N$ & Partial-asynchronism bound $\max\{s_{\mathrm{upd}},\tau_{\max}+1\}$. \\
$u$ & $\mathbb R_+$ & Mean number of local updates per node. \\
$d_{\mathrm{out}}$ & $\mathbb R_+$ & Relative Frobenius discrepancy between asynchronous and synchronous predictions; write $d_{\mathrm{out}}(u)$ when emphasizing work $u$. \\
$\alpha$ & $(0,1)$ & PPR teleport parameter. \\
\bottomrule
\end{tabularx}
\end{table*}

\clearpage
\section{Extended Background and Related Work}
\label{app:extended-background}

\subsection{Well-posed implicit graph neural networks}
\label{app:implicit-gnns}

For fixed $(X,G)$, well-posedness of
$H^\star=\mathcal T_\Theta(H^\star;X,G)$ means existence and uniqueness
of $H^\star$. It does not by itself imply that the Picard iterates converge:
this follows, for example, when
$\mathcal T_\Theta(\cdot;X,G)$ is a contraction on a complete metric
space. Conversely, an equilibrium may be computed by a convergent solver other
than Picard iteration. Once a solution has been obtained and
$I-J_H\mathcal T_\Theta(H^\star;X,G)$ is nonsingular, the implicit
function theorem gives
\[
    \frac{\partial H^\star}{\partial\Theta}
    =
    \left(
        I-J_H\mathcal T_\Theta(H^\star;X,G)
    \right)^{-1}
    \frac{\partial\mathcal T_\Theta}{\partial\Theta}
        (H^\star;X,G),
\]
allowing gradients to be computed without storing all forward iterates
~\citep{bai2019deep}. This differentiation condition is local and should not be
confused with a global existence, uniqueness, or solver-convergence guarantee.

The original recurrent GNN already defined node states through a fixed-point
transition and imposed a contraction requirement to make the construction
well defined~\citep{scarselli2008graph}. Translating IGNN's original
column-major presentation to the row-major convention used in this paper gives
\[
    H=\phi\!\left(AHW+B_\theta(X)\right),
\]
where $A$ encodes the graph and $\phi$ is componentwise nonexpansive. A
sufficient well-posedness condition is
\[
    \lambda_{\mathrm{pf}}\!\left(
        \left|W^\top\otimes A\right|
    \right)<1,
\]
which is enforced in practice through a more tractable induced-norm constraint
on $W$~\citep{gu2020implicit}. EIGNN removes the intermediate nonlinearity and
constrains the Frobenius norm of its learnable transformation, yielding
tractable closed-form forward and backward solutions~\citep{liu2021eignn}.
Convergent Graph Solvers instead construct
input-dependent linear contracting maps~\citep{park2021convergent}, while
MIGNN replaces direct contractivity by a strongly monotone residual operator
and uses monotone or scaled-orthogonal parameterizations together with operator
splitting~\citep{baker2023implicit}.

The representative routes can be summarized as follows; these are sufficient
architectural mechanisms used by the corresponding models, not necessary
conditions for implicit GNNs in general.
\begin{description}
    \item[Contractive transition.]
    The original GNN and IGNN constrain the transition or recurrent operator in
    a norm or Perron--Frobenius spectral radius so that Picard iteration
    contracts~\citep{scarselli2008graph,gu2020implicit}.

    \item[Linear equilibrium.]
    EIGNN and Convergent Graph Solvers use normalized or input-dependent linear
    recurrences with an analytically controlled equilibrium
    ~\citep{liu2021eignn,park2021convergent}.

    \item[Monotone operator.]
    MIGNN requires a strongly monotone residual and uses monotone or
    scaled-orthogonal parameterizations together with operator splitting
    ~\citep{baker2023implicit}.

    \item[Convex energy.]
    GIND and EnergyGNN define $H^\star$ as the minimizer of a convex or
    strongly convex energy implemented as a partially input convex neural network~\citep{amos2017input}; the admissible energy architecture and optimization
    solver determine the guarantee
    ~\citep{chen2022optimization,solodova2025graph}.

    \item[Subhomogeneous update.]
    SubDEQ controls the contraction degree through positive subhomogeneous
    nonlinearities while leaving the recurrent weights spectrally unconstrained
    ~\citep{sittoni2024subhomogeneous}.
\end{description}

Optimization-based models define the representation as
\[
    H^\star=\operatorname*{arg\,min}_H E_\theta(H;X,G).
\]
Strong convexity of $E_\theta(\cdot;X,G)$ ensures a unique minimizer,
while convergence still depends on the optimization algorithm used to compute
it. GIND relates its implicit nonlinear diffusion to the minimizer of a convex
graph energy~\citep{chen2022optimization}. EnergyGNN instead learns a
node-separable strongly convex energy: both its message and update energies are
implemented by partially input-convex neural networks with the hidden
representations designated as convex inputs~\citep{solodova2025graph}. This
separability supports local asynchronous optimization under bounded update and
communication delays. These constructions are reviewed more broadly by
\citet{yang2025implicit}.

\subsection{Subhomogeneous equilibrium theory}
\label{app:subdeq}

Let $T:\mathbb R^a\to\mathbb R^b$ be Lipschitz and let
$\Omega\subseteq\operatorname{dom}_{+}(T)$. Following
\citet{sittoni2024subhomogeneous}, $T$ is $\mu$-subhomogeneous on $\Omega$ when
\[
    |Mz|\leq\mu T(z),
    \qquad z\in\Omega,\quad M\in\partial T(z),
\]
and is \emph{strongly} $\mu$-subhomogeneous when
\[
    |M|\,|z|\leq\mu T(z),
    \qquad z\in\Omega,\quad M\in\partial T(z).
\]
All absolute values and inequalities are entrywise. Strong subhomogeneity is
preserved by the compositions used in SubDEQ, while ordinary subhomogeneity
need not be. The definition also extends Euler's identity: if $T$ is
differentiable and positively $\eta_{\mathrm{hom}}$-homogeneous, then
$J_T(z)z=\eta_{\mathrm{hom}}T(z)$, and hence $T$ is
$\eta_{\mathrm{hom}}$-subhomogeneous wherever $T(z)\geq0$.

The interior of the positive orthant is complete under the Thompson metric
\[
    \delta_{\mathrm T}(x,y)
    =
    \|\log x-\log y\|_\infty .
\]
If $T:\mathbb R_{++}^a\to\mathbb R_{++}^b$ is positive, Lipschitz, and
$\mu$-subhomogeneous, then
\[
    \delta_{\mathrm T}(T(x),T(y))
    \leq
    \mu\,\delta_{\mathrm T}(x,y).
\]
Equivalently, $\log\circ T\circ\exp$ is $\mu$-Lipschitz in
$\ell_\infty$. When $a=b$ and $\mu<1$, Banach's theorem therefore gives a
unique fixed point, and Picard iteration converges to it globally and linearly.

SubDEQ applies these results to the weight-tied, input-injected update
\[
    z^{(k+1)}
    =
    \sigma_1\!\left(
        \sigma_2(Wz^{(k)})+B_\theta(x)
    \right).
\]
The linear map is positively $1$-homogeneous, the nonnegative input injection
is a translation, and suitable positive subhomogeneous activations determine
the overall contraction degree. This construction does not impose a spectral
constraint on $W$. For example,
$t\mapsto\tanh(t)+1.2$ is $0.99$-subhomogeneous
~\citep{sittoni2024subhomogeneous}.

For graphs, \citet{sittoni2024subhomogeneous} instantiate the recurrent map
with a fixed APPNP operator. With $\widehat A^+$ denoting the symmetric
normalized adjacency with self-loops and $\alpha\in(0,1)$, one of their nonlinear
variants is
\[
    \begin{aligned}
    \widetilde H^\star
        &=
        \tanh\!\left(
            (1-\alpha)\widehat A^+\widetilde H^\star
        \right)
        +\alpha B_\theta(X)
        +1.2\,\mathbf 1,\\
    \widehat Y&=\operatorname{softmax}(\widetilde H^\star).
    \end{aligned}
\]
The graph enters this model through the fixed linear operator
$\widehat A^+$; it is not inferred from the evolving state. SheafDEQ
replaces this fixed propagation by adaptive propagation
$\mathcal P_\Theta(H;G)$ defined in Section~\ref{sec:adaptive-equilibrium}.

\subsection{Cellular sheaves, sections, and neural sheaf diffusion}
\label{app:sheaves}

Fix an arbitrary orientation of each edge. For an oriented edge $e=(u,v)$,
the sheaf coboundary compares the endpoint signals after mapping them into the
same edge stalk:
\[
    (\delta_{\mathcal F}\mathbf H)_e
    =
    \mathcal F_{v\triangleleft e}\mathbf H_v
    -
    \mathcal F_{u\triangleleft e}\mathbf H_u .
\]
Changing the chosen orientation only changes the sign of the corresponding
row of $\delta_{\mathcal F}$ and leaves
$L_{\mathcal F}=\delta_{\mathcal F}^{\top}\delta_{\mathcal F}$ unchanged.
For $e=(u,v)$, its blocks are
\[
    (L_{\mathcal F})_{vv}
    =
    \sum_{e\ni v}
    \mathcal F_{v\triangleleft e}^{\top}
    \mathcal F_{v\triangleleft e},
    \qquad
    (L_{\mathcal F})_{uv}
    =
    -
    \mathcal F_{u\triangleleft e}^{\top}
    \mathcal F_{v\triangleleft e}.
\]
For the connected graph $G$, let
$D_G=\operatorname{diag}(\deg(v):v\in V)$. The scalar-degree-normalized
sheaf Laplacian used throughout this paper is
\[
    \Delta_{\mathcal F}
    =
    (D_G^{-1/2}\otimes I_d)L_{\mathcal F}
    (D_G^{-1/2}\otimes I_d),
    \qquad
    \mathcal P_{\mathcal F}=I_{nd}-\Delta_{\mathcal F}.
\]
Consequently, the off-diagonal blocks of
$\mathcal P_{\mathcal F}=I_{nd}-\Delta_{\mathcal F}$ propagate information through
matrix products of restriction maps rather than through nonnegative scalar
edge weights.

\rev{We use ``hard sheaf consensus'' to mean exact agreement in the shared
edge spaces: neighbouring representations agree after applying their
restriction maps, without needing to be identical at the nodes.}
A global section is a $0$-cochain \rev{$h\in\mathbb R^{nd}$} satisfying
\[
    \mathcal F_{v\triangleleft e}\rev{h_v}
    =
    \mathcal F_{u\triangleleft e}\rev{h_u}
    \quad\text{for every }e=(u,v).
\]
The space of global sections is
$H^0(G;\mathcal F)=\ker\delta_{\mathcal F}
=\ker L_{\mathcal F}$\rev{~\citep{hansen2019toward}}. Scalar degree normalization gives
\[
    \ker\Delta_{\mathcal F}
    =
    (D_G^{1/2}\otimes I_d)\ker L_{\mathcal F},
\]
the degree-rescaled global-section space. For a fixed sheaf, continuous
scalar-degree-normalized diffusion
satisfies
\[
    \dot{\mathbf H}(t)=-\Delta_{\mathcal F}\mathbf H(t),
    \qquad
    \mathbf H(t)=e^{-t\Delta_{\mathcal F}}\mathbf H(0),
\]
and converges to the orthogonal projection of $\mathbf H(0)$ onto
$\ker\Delta_{\mathcal F}$\rev{~\citep{hansen2021opinion}}.
\rev{Applying $D_G^{-1/2}\otimes I_d$ to this limit gives a global section
in each channel, and the rescaled limit therefore satisfies hard sheaf consensus. Suitable sheaves
can thus preserve structured, nonuniform limiting representations without
requiring identical node representations}
~\citep{bodnar2022neural}. \rev{The} unit-step discrete operator
$\rev{I_{nd}}-\Delta_{\mathcal F}$ may have eigenvalue $-1$;
unconditional convergence of its powers would require an additional step-size
or laziness condition. \rev{More generally, discrete-time diffusion satisfies
\[
    \mathbf H^{(k+1)}
    =(I_{nd}-\eta\Delta_{\mathcal F})\mathbf H^{(k)}.
\]
For $\eta>0$ with $\eta\lambda_{\max}(\Delta_{\mathcal F})<2$, it converges to
the same projection and hence, after degree rescaling, to hard sheaf consensus.}

Neural Sheaf Diffusion makes the restriction maps layer- and state-dependent.
A generic incidence map can be written as
\[
    \mathcal F_{v\triangleleft e}^{(k)}
    =
    g_{\psi_k}
    \!\left(H_v^{(k)},H_u^{(k)}\right)
    \in\mathbb R^{d\times d},
    \qquad e=(u,v),
\]
which induces
$\mathcal F_{\psi_k}(H^{(k)};G)$ and therefore
$\Delta_{\mathcal F_{\psi_k}(H^{(k)};G)}$
~\citep{bodnar2022neural}. In a finite explicit network, different layers may
use different parameters $\psi_k$. In an equilibrium model, one
weight-tied state $H$ determines the restriction maps, those maps determine
the propagation operator, and that operator acts again on $H$. Hence the
coupled map
\[
    H\longmapsto
    \mathcal P_\Theta(H;G)
    =
    \mathscr R^{-1}\!\left(
    \mathcal P_{\mathcal F_\psi(H;G)}
    (I_n\otimes W_1)\mathscr R(H) W_2
    \right)
\]
is generally nonlinear even before an activation is applied.

\subsection{Extended related work}
\label{app:related-work}

Beyond NSD, finite-depth sheaf architectures extend the propagation mechanism
through attention~\citep{barbero2022sheafattention}, nonlinear sheaf
Laplacians~\citep{zaghen2024sheaf}, flat-bundle global communication
~\citep{bamberger2024bundle}, directed communication
~\citep{ribeiro2025cooperative}, and polynomial sheaf filters
~\citep{borgi2025polynomial}. DNSD is the most direct explicit-depth
comparison: it combines sheaf-adjacency aggregation with normalization and
gating to sustain propagation at high finite depth, but still computes a
prescribed composition rather than an equilibrium~\citep{bourgerie2026deep}.

The implicit GNNs reviewed above obtain well-posed representations through
contractive recurrent dynamics~\citep{scarselli2008graph,gu2020implicit},
controlled linear equilibria~\citep{liu2021eignn,park2021convergent}, or
monotone operators~\citep{baker2023implicit}. These models do not couple the
equilibrium state to learned cellular-sheaf restriction maps. Relative to
SubDEQ, our direct theoretical basis, SheafDEQ retains the positive
subhomogeneous contraction guarantee but replaces fixed scalar graph
propagation operator with a sheaf operator inferred from the equilibrium representation
~\citep{sittoni2024subhomogeneous}.

Optimization-induced models expose a complementary distinction. GIND derives
an explicit convex objective for its nonlinear incidence-based diffusion, but
retains a fixed scalar graph differential~\citep{chen2022optimization}.
EnergyGNN instead learns a strongly convex energy from local
PICNN-parameterized message and update terms and supports partially
asynchronous inference~\citep{solodova2025graph}. Both remain variational:
their representations minimize scalar energies. SheafDEQ need not derive
from a scalar potential and instead defines the well-posed fixed point of
adaptive, matrix-valued sheaf propagation.

Recent work on asynchronous nonlinear sheaf diffusion studies distributed
gradient descent on the Dirichlet energy of a fixed coordination sheaf and
proves linear convergence to its minimizer set under strongly convex, smooth
edge potentials and bounded communication delays~\citep{zhao2025asynchronous}.
In contrast, SheafDEQ learns state-dependent restriction maps and need not
arise from a scalar energy. Moreover, its contractive update has a unique
equilibrium, so every admissible bounded-staleness schedule converges to the
same representation, whereas energy minimization over a non-singleton solution
set can select different minimizers.

\section{Additional theory and proofs}
\label{app:proofs}

\subsection{Proof of Proposition~\ref{prop:orthant-routing}}
\label{app:proof-orthant-routing}

\begin{proof}
Set \(k=\lfloor d/2\rfloor\) and choose distinct sign vectors
\(s_1,\ldots,s_C\in\{-1,+1\}^d\), each with exactly \(k\) negative
entries. For every pair \(a,b\), their equal numbers of negative entries
imply that there is a
permutation \(\pi_{ab}\) satisfying
\[
    s_{b,\pi_{ab}(i)}=s_{a,i}.
\]
Define the signed permutation matrix \(Q_{ab}\) by
\[
    (Q_{ab})_{i,\pi_{ab}(i)}=s_{a,i}.
\]
For every \(x\succ0\),
\begin{equation}
    \operatorname{sign}(Q_{ab}x)=s_a,
    \qquad
    \operatorname{sign}(Q_{ab}^{\top}x)=s_b.
    \label{eq:routing-signs}
\end{equation}

Orient each edge \(e=\{v,u\}\) arbitrarily, write
\(a=\ell(v)\), \(b=\ell(u)\), and set
\begin{equation}
    \mathcal F_{v\triangleleft e}=d^{-1/2}I_d,
    \qquad
    \mathcal F_{u\triangleleft e}=d^{-1/2}Q_{ab}.
    \label{eq:routing-restrictions}
\end{equation}
Both restriction maps are invertible and have Frobenius norm one. Direct
substitution into the scalar-degree-normalized sheaf Laplacian gives
\begin{equation}
 (\mathcal P_{\mathcal F_\ell})_{vv}
     =\left(1-\frac1d\right)I_d,
 \qquad
 (\mathcal P_{\mathcal F_\ell})_{vu}
     =\frac{Q_{vu}}{d\sqrt{d_vd_u}},
 \label{eq:routing-blocks}
\end{equation}
where \(Q_{uv}=Q_{vu}^{\top}\), and the row signs of \(Q_{vu}\) are
\(s_{\ell(v)}\).

We first establish existence and uniqueness. Define the symmetric block
matrix
\[
    (A_Q)_{vu}=\frac{Q_{vu}}{\sqrt{d_vd_u}}
    \quad (u\sim v),
    \qquad
    (A_Q)_{vv}=0.
\]
For every block vector \(x=(x_v)_{v\in V}\), orthogonality of the signed
permutations and
\(2|p^\top q|\leq\|p\|_2^2+\|q\|_2^2\) give
\[
\begin{aligned}
    |x^\top A_Qx|
    &\leq
    \sum_{\{v,u\}\in E}
    \left(
        \frac{\|x_v\|_2^2}{d_v}
        +\frac{\|x_u\|_2^2}{d_u}
    \right)
    =\|x\|_2^2.
\end{aligned}
\]
Thus \(\operatorname{spec}(A_Q)\subseteq[-1,1]\), and
Eq.~\eqref{eq:routing-blocks} implies
\begin{equation}
    0\preceq\mathcal P_{\mathcal F_\ell}
    =
    \left(1-\frac1d\right)I_{nd}+\frac1dA_Q
    \preceq I_{nd}.
    \label{eq:routing-spectrum}
\end{equation}

The update maps the compact box
\([\beta-1,\beta+1]^{nd}\) into itself, so Brouwer's theorem gives a fixed
point. Suppose that \(h\) and \(g\) are two fixed points. Set
\[
    z=h-\beta\mathbf1,\qquad
    w=g-\beta\mathbf1,\qquad
    \delta=h-g=z-w.
\]
Since \(z,w\in(-1,1)^{nd}\),
\[
    \operatorname{arctanh}(z)=\mathcal P_{\mathcal F_\ell}h,
    \qquad
    \operatorname{arctanh}(w)=\mathcal P_{\mathcal F_\ell}g.
\]
If \(\delta\neq0\), strict expansiveness of
\(\operatorname{arctanh}\), together with
Eq.~\eqref{eq:routing-spectrum}, yields
\[
 \|\delta\|_2^2
 <
 \delta^\top\!
 \left(\operatorname{arctanh}(z)-\operatorname{arctanh}(w)\right)
 =
 \delta^\top\mathcal P_{\mathcal F_\ell}\delta
 \leq
 \|\delta\|_2^2,
\]
a contradiction. Hence the fixed point \(h^\star\) is unique.

Because \(\beta>1\), every coordinate of \(h^\star\) is positive. Let
\(p=\mathcal P_{\mathcal F_\ell}h^\star\) and fix a node \(v\) with
\(a=\ell(v)\). Equations~\eqref{eq:routing-signs} and
\eqref{eq:routing-blocks} give
\[
    p_v=\alpha h_v^\star+r_v,
    \qquad
    \alpha=1-\frac1d,
    \qquad
    r_v=
    \sum_{u\sim v}
    \frac{Q_{vu}h_u^\star}{d\sqrt{d_vd_u}},
\]
where
\[
    (r_v)_i>0\quad\text{if }s_{a,i}=+1,
    \qquad
    (r_v)_j<0\quad\text{if }s_{a,j}=-1.
\]
Choose such coordinates \(i,j\), and write
\[
    x=h_{v,i}^\star-h_{v,j}^\star,\qquad
    y=p_{v,i}-p_{v,j},\qquad
    m=(r_v)_i-(r_v)_j>0.
\]
Then \(y=\alpha x+m\) and
\[
    x=\tanh(p_{v,i})-\tanh(p_{v,j}).
\]
If \(x\leq0\), monotonicity of \(\tanh\) implies \(y\leq0\), while its
\(1\)-Lipschitz property implies \(x\geq y\). Consequently,
\[
    y=\alpha x+m\geq\alpha y+m,
\]
contradicting \(y\leq0\), \(1-\alpha>0\), and \(m>0\). Therefore
\[
    h_{v,i}^\star>h_{v,j}^\star
\]
for every coordinate \(i\) marked \(+1\) and coordinate \(j\) marked
\(-1\) by \(s_a\).

Finally, fix \(r\neq a\) and define
\[
 A=\{i:s_{a,i}=+1,\ s_{r,i}=-1\},
 \qquad
 B=\{j:s_{a,j}=-1,\ s_{r,j}=+1\}.
\]
The distinct sign vectors have equal numbers of negative entries, so
\(|A|=|B|>0\). The established
coordinate ordering gives
\[
\begin{aligned}
 (s_a-s_r)^\top h_v^\star
 &=
 2\left(
     \sum_{i\in A}h_{v,i}^\star
     -
     \sum_{j\in B}h_{v,j}^\star
   \right)
 >0.
\end{aligned}
\]
Thus \(s_a^\top h_v^\star>s_r^\top h_v^\star\) for every \(r\neq a\),
which proves the classification claim.
\end{proof}

\subsection{\rev{Nonlinear equilibria and soft sheaf consensus}}
\label{app:equilibrium-interpretation}

\rev{SheafDEQ repeatedly applies sheaf propagation followed by a shifted
$\tanh$ nonlinearity and a fixed input injection
(Section~\ref{sec:adaptive-equilibrium}). At equilibrium, the learned
restriction maps define a sheaf, but convergence of the complete update does
not require the representation to lie in the normalized Laplacian kernel of
that sheaf. This differs from fixed-sheaf diffusion, which, under the conditions
in Appendix~\ref{app:sheaves}, converges to the degree-rescaled global-section
space~\citep{hansen2021opinion}. Global sections
express exact agreement after neighbouring representations are mapped into
their shared edge spaces. The limiting diffusion therefore reaches hard
sheaf consensus through projection onto this space.}

\rev{Propagation with persistent input gives a useful contrast. A PPR-style
recurrence on a fixed sheaf~\citep{gasteiger2018predict} takes the form
\[
    h^{(k+1)}=(1-\alpha)\mathcal P_{\mathcal F}h^{(k)}+\alpha s,
    \qquad 0<\alpha<1,
\]
where $h^{(k)},s\in\mathbb R^{nd}$ and $s$ is the injected signal. When the
recurrence converges, its
infinite-iteration limit is the resolvent
\[
\begin{aligned}
    h^\star
    &=\alpha[I_{nd}-(1-\alpha)\mathcal P_{\mathcal F}]^{-1}s\\
    &=\bigl[I_{nd}+\tfrac{1-\alpha}{\alpha}\Delta_{\mathcal F}\bigr]^{-1}s.
\end{aligned}
\]
Rather than projecting onto hard sheaf consensus, this resolvent provides soft
low-pass filtering. We use ``soft sheaf consensus'' intuitively to describe
balancing agreement in the shared edge spaces with retention of the input
signal, rather than requiring exact agreement. Soft sheaf consensus similarly
balances sheaf consistency with local
objectives, with better conditioning identified as a practical benefit for
distributed computation~\citep{seely2026learning}. In this sense, SheafDEQ's
shifted $\tanh$ update and persistent input support a soft sheaf-consensus
interpretation rather than imposing hard sheaf consensus.}

\subsection{\rev{Proof of Theorem~\ref{thm:sheafdeq}}}
\label{app:proof-3.2}

For the remaining fixed-point proofs, we use the following notation.
Fix \(p\in[1,\infty]\). \rev{Throughout the remaining fixed-point proofs},
\(H\in\mathbb R_{++}^{n\times c}\), with
\(\mathbf H:=\mathscr R(H)\in\mathbb R_{++}^{nd\times q}\), and
matrix-valued maps are identified with
their vectorizations when applying the definitions of
\citet{sittoni2024subhomogeneous}. In particular, a Lipschitz positive map \(F\)
is \(\mu\)-subhomogeneous on the positive cone if
\[
|M\operatorname{vec}(H)|
\leq
\mu\,\operatorname{vec}(F(H))
\]
for every \(H>0\) and every \(M\in\partial F(H)\). It is strongly
\(\mu\)-subhomogeneous when the left-hand side is replaced by
\(|M|\,|\operatorname{vec}(H)|\).

\begin{lemma}[Pre-restriction normalization is scale invariant]
\label{lem:rowwise-scale}
For every \(s>0\) and every state $H$ whose rows in $\mathscr R(H)$ are
nonzero,
\[
\operatorname{norm}_p\!\left(\mathscr R(sH)\right)
=\operatorname{norm}_p\!\left(\mathscr R(H)\right).
\]
\end{lemma}

\begin{proof}
For every row $i$ of the sheaf-stacked representation,
\[
\big[\operatorname{norm}_p(\mathscr R(sH))\big]_{i,:}
=
\frac{s[\mathbf H]_{i,:}}{\|s[\mathbf H]_{i,:}\|_p}
=
\frac{s[\mathbf H]_{i,:}}{s\|[\mathbf H]_{i,:}\|_p}
=
\big[\operatorname{norm}_p(\mathscr R(H))\big]_{i,:}.
\]
\end{proof}

\begin{lemma}[Scale invariance of the restriction maps]
\label{lem:restriction-scale}
For every \(s>0\) and \(H\in\mathbb R_{++}^{n\times c}\),
\[
\mathcal F_\psi(sH;G)=\mathcal F_\psi(H;G),
\qquad
\mathcal P_{\mathcal F_\psi(sH;G)}
=
\mathcal P_{\mathcal F_\psi(H;G)}.
\]
\end{lemma}

\begin{proof}
For an undirected edge $e=\{u,v\}$ with stored orientation $u\to v$,
Lemma~\ref{lem:rowwise-scale} shows that the complete normalized
sheaf-stacked representations, and hence their node blocks $u$ and $v$, are
unchanged by scaling.
The optional edge feature $a_e$ is independent of $H$. Substitution into both
direction-dependent generators, followed by the same coefficient
normalization $\Pi$, therefore gives
\[
\mathcal F_{u\triangleleft e}(sH)=\mathcal F_{u\triangleleft e}(H),
\qquad
\mathcal F_{v\triangleleft e}(sH)=\mathcal F_{v\triangleleft e}(H).
\]
Thus every incidence map, and hence the complete sheaf, is unchanged. Because
the graph is connected, the fixed scalar normalization $D_G^{-1/2}$ is well
defined. Hence unchanged incidence maps imply unchanged
$L_{\mathcal F}$, $\Delta_{\mathcal F}$, and
$\mathcal P_{\mathcal F}=I_{nd}-\Delta_{\mathcal F}$.
\end{proof}

\begin{lemma}[Homogeneity of the sheaf propagation block]
\label{lem:sheaf-homogeneous}
The adaptive propagation $\mathcal P_\Theta(\cdot;G)$ is positively
$1$-homogeneous:
\[
\mathcal P_\Theta(sH;G)=s\mathcal P_\Theta(H;G)
\qquad\text{for every }s>0.
\]
\end{lemma}

\begin{proof}
Using Lemma~\ref{lem:restriction-scale} and linearity in the state,
\begin{align*}
\mathcal P_\Theta(sH;G)
&=
\mathscr R^{-1}\!\left(
\mathcal P_{\mathcal F_\psi(sH;G)}
(I_n\otimes W_1)\mathscr R(sH)W_2\right)\\
&=
\mathscr R^{-1}\!\left(
\mathcal P_{\mathcal F_\psi(H;G)}
(I_n\otimes W_1)s\mathscr R(H)W_2\right)
=
s\mathcal P_\Theta(H;G).
\end{align*}
No sign or spectral condition on \(W_1\), \(W_2\), or the restriction maps is
used.
\end{proof}

\paragraph{Regularity scope.}
The preceding lemmas establish scale invariance and positive homogeneity, but
do not by themselves establish Lipschitz continuity of the adaptive
propagation. Accordingly, Lipschitz continuity of
\(\mathcal P_\Theta(\cdot;G)\) remains an explicit regularity assumption in
Theorem~\ref{thm:sheafdeq}. Nonzero generator outputs ensure that
\(\Pi(A)=A/\|A\|_{\mathrm F}\) is pointwise well defined, but pointwise
nonvanishing alone does not provide a uniform Lipschitz bound for the
normalized generators or for the resulting adaptive propagation. This property
is not explicitly enforced or certified during training; the theorem
therefore applies to learned parameter settings for which the stated
regularity assumption holds. The continued-iteration experiments evaluate the
resulting numerical behaviour but do not verify this assumption.

\begin{lemma}[Subhomogeneity of the complete update]
\label{lem:complete-subhom}
Assume that $\mathcal P_\Theta(\cdot;G)$ is Lipschitz, and that the scalar map
$\sigma_{\beta_{\mathrm{shift}}}(t)=\tanh(t)+\beta_{\mathrm{shift}}$ is
Lipschitz, positive, and strongly $\mu$-subhomogeneous on $\mathbb R$. If
$B_\theta(X)\geq0$ entrywise, then
\[
\mathcal T_\Theta(H;X,G)
=
\sigma_{\beta_{\mathrm{shift}}}\!\left(
    \mathcal P_\Theta(H;G)
\right)+B_\theta(X)
\]
is positive, Lipschitz, and \(\mu\)-subhomogeneous on the positive cone.
\end{lemma}

\begin{proof}
Define the entrywise outer map
\[
\Sigma_{\beta_{\mathrm{shift}}}(Z)
:=
\tanh(Z)+\beta_{\mathrm{shift}}\mathbf 1_{n\times c}.
\]
By assumption, the scalar map
$\sigma_{\beta_{\mathrm{shift}}}$ is positive and strongly
$\mu$-subhomogeneous. Since $\Sigma_{\beta_{\mathrm{shift}}}$ acts
entrywise, its generalized Jacobians are diagonal, so the scalar property
lifts to the matrix-valued map: $\Sigma_{\beta_{\mathrm{shift}}}$ is
positive and strongly \(\mu\)-subhomogeneous, hence also
\(\mu\)-subhomogeneous.

Apply \citet[Lemma~4.1]{sittoni2024subhomogeneous} with the homogeneous inner map
$\mathcal P_\Theta(\cdot;G)$, homogeneity degree $1$, and outer map
$\Sigma_{\beta_{\mathrm{shift}}}$. Its first composition
rule gives
\[
\Sigma_{\beta_{\mathrm{shift}}}\circ\mathcal P_\Theta
\in\operatorname{subhom}_{\mu}
\bigl(\mathbb R_{++}^{n\times c}\bigr).
\]
This is the central composition step: strong subhomogeneity is required only of
the entrywise outer branch, while the complete signed adaptive operator is
concluded to be ordinarily \(\mu\)-subhomogeneous.

Finally, set \(b_X:=B_\theta(X)\geq0\). The complete update is
\[
\mathcal T_\Theta(H;X,G)
=
(\Sigma_{\beta_{\mathrm{shift}}}\circ\mathcal P_\Theta)(H)+b_X.
\]
The constant \(b_X\) contributes no generalized-Jacobian term. Therefore, for
every
$M\in\partial(\Sigma_{\beta_{\mathrm{shift}}}\circ\mathcal P_\Theta)(H)$,
\[
|M\operatorname{vec}(H)|
\leq
\mu\,\operatorname{vec}\bigl(
(\Sigma_{\beta_{\mathrm{shift}}}\circ\mathcal P_\Theta)(H)
\bigr)
\leq
\mu\,\operatorname{vec}\bigl(\mathcal T_\Theta(H;X,G)\bigr).
\]
Hence $\mathcal T_\Theta(\cdot;X,G)$ is $\mu$-subhomogeneous. It is
positive because the shifted outer map is positive and $b_X\geq0$, and it is
Lipschitz because the adaptive propagation and scalar branch are Lipschitz and
$b_X$ is constant.
\end{proof}

For the shifted-\(\tanh\) family, the scalar condition is
\[
|t|\operatorname{sech}^2(t)
\leq
\mu\bigl(\tanh(t)+\beta_{\mathrm{shift}}\bigr).
\]
The value $\beta_{\mathrm{shift}}=1.2,\mu=0.99$ is reported by
\citet{sittoni2024subhomogeneous} and supports the equilibrium update used in
the experiments.

\begin{proof}[Completion of the proof of Theorem~\ref{thm:sheafdeq}]
By Lemmas~\ref{lem:restriction-scale}, \ref{lem:sheaf-homogeneous}, and
\ref{lem:complete-subhom},
\(\mathcal T_\Theta(\cdot;X,G)\) is positive, Lipschitz, and
\(\mu\)-subhomogeneous on the positive cone. After vectorizing the state, the
claim follows directly from
\citet[Theorem~3.7]{sittoni2024subhomogeneous}, including uniqueness and global
linear convergence of the Picard iteration in the entrywise infinity norm,
where $\|H\|_\infty=\max_{i,j}|H_{ij}|$.
\end{proof}

\subsection{Proof of Corollary~\ref{cor:asynchronous}}
\label{app:asynchronous-proof}

\begin{proof}
Let
\[
    \mathcal C_X
    :=
    \left\{
        U : e^U\in\mathcal K_X
    \right\},
    \qquad
    \mathcal Q_\Theta(U)
    :=
    \log\mathcal T_\Theta(e^U;X,G).
\]
By Theorem~\ref{thm:sheafdeq},
\[
    \left\|
        \mathcal Q_\Theta(U)-\mathcal Q_\Theta(\widehat U)
    \right\|_\infty
    \leq
    \mu
    \left\|
        U-\widehat U
    \right\|_\infty,
    \qquad
    U,\widehat U\in\mathcal C_X,
\]
for some \(0<\mu<1\). Moreover,
\(U^\star:=\log H^\star\) is the unique fixed point of \(\mathcal Q_\Theta\).

Whenever a node updates, its new state belongs to its corresponding block of
\(\mathcal K_X\), independently of the states used in the update. Since every
node updates within \(B_{\mathrm{async}}\) steps, \(H(t)\in\mathcal K_X\) for
all \(t\geq B_{\mathrm{async}}\).
The bounded-delay condition then ensures that every stale view used after time
\(2B_{\mathrm{async}}-1\) also belongs to \(\mathcal K_X\).

For \(t\geq2B_{\mathrm{async}}-1\), define
\[
    e_i(t)
    :=
    \left\|U_i(t)-U_i^\star\right\|_\infty,
    \qquad
    E_{\max}(t)
    :=
    \max_{\substack{t-B_{\mathrm{async}}+1\leq s\leq t\\1\leq i\leq n}}
    e_i(s).
\]
If node \(i\) updates at time \(t\), contractivity gives
\[
    e_i(t+1)
    \leq
    \mu
    \max_j e_j\!\left(\tau_j^i(t)\right)
    \leq
    \mu E_{\max}(t).
\]
If it does not update, \(e_i(t+1)=e_i(t)\). Consequently, \(E_{\max}(t)\) is
nonincreasing.

During the interval
\(\{t,\ldots,t+B_{\mathrm{async}}-1\}\), every node updates at least once.
Hence, from time \(t+B_{\mathrm{async}}\) onward, every current block has
error at most \(\mu E_{\max}(t)\). Once the preceding values have also left
the delay window, we
obtain
\[
    E_{\max}(t+2B_{\mathrm{async}}-1)\leq\mu E_{\max}(t).
\]
Iterating this inequality yields
\[
    E_{\max}\!\left(t+\nu(2B_{\mathrm{async}}-1)\right)
    \leq
    \mu^\nu E_{\max}(t)
    \longrightarrow 0.
\]
Therefore \(U_i(t)\to U_i^\star\) for every node \(i\), and thus
\(H(t)\to H^\star\).
\end{proof}

\section{\rev{Comparison Architectures}}
\label{app:baselines}

This appendix documents the implementation of every \rev{comparison model} used in our comparison
against SheafDEQ. Throughout, $H\in\mathbb R^{n\times c}$ denotes node \rev{representations},
$X\in\mathbb R^{n\times f}$ node features, and $A$ is the raw adjacency matrix.
We write $D=\operatorname{diag}(A\mathbf1)$ and
$\widehat A=D^{-1/2}AD^{-1/2}$. For models that add self-loops, let
$A^+=A+I$, $D^+=\operatorname{diag}(A^+\mathbf1)$, and
$\widehat A^+=(D^+)^{-1/2}A^+(D^+)^{-1/2}$. All
models share the training protocol (Adam, warmup-stable-decay learning-rate schedule,
gradient clipping at norm~$1$, cross-entropy or MSE loss) except where a model requires
a different inner solver, noted explicitly below.

\subsection{Feature-only MLP}
A depth-parity MLP with no access to $A$, used as a graph-agnostic lower
bound. It has $K$ total linear layers: an input projection
$\mathbb R^f\to\mathbb R^c$, $K-2$ hidden maps
$\mathbb R^c\to\mathbb R^c$, and an output projection
$\mathbb R^c\to\mathbb R^o$. ReLU and dropout
$p_{\mathrm{drop}}=0.5$ are applied between hidden layers. It is trained by
ordinary backpropagation.

\subsection{APPNP family}
\label{app:appnp}
Following Sittoni--Tudisco, verified against the authors' released code
(\texttt{COMPiLELab/SubDEQ}). An encoder $H^{(0)}=B_\theta(X)$---a two-layer
MLP of hidden width $c$ with ReLU and dropout $p_{\mathrm{drop}}=0.5$ applied
to the raw input and the post-ReLU hidden representation---is propagated for
a fixed $K$ steps with teleport weight $\alpha$:
\[
\begin{aligned}
H^{(0)}&=B_\theta(X),\\
U^{(k)}&=(1-\alpha)\widehat A^+H^{(k)}+\alpha H^{(0)},\\
H^{(k+1)}&=
\begin{cases}
U^{(k)}, & \textsc{appnp},\\[2pt]
\tanh(U^{(k)})+\beta_{\mathrm{shift}},
& \textsc{shifted-tanh},
\end{cases}
\end{aligned}
\]
All variants use \(\alpha=0.1\) and the dataset-dependent iteration budget
\(K\) reported in Table~\ref{tab:model-configs}. We evaluate standard APPNP
and shifted-$\tanh$ variants with
\(\beta_{\mathrm{shift}}\in\{1.25,1.65\}\). Reported results use the finite
\(K\)-step unroll and ordinary backpropagation through its iterations. The
standard APPNP recurrence also admits the closed-form equilibrium
\[
H^\star
=
\alpha\bigl[I-(1-\alpha)\widehat A^+\bigr]^{-1}H^{(0)}.
\]
This expression is implemented as an alternate diagnostic solver and agrees
with 200-step Picard iteration to approximately \(10^{-6}\). A linear readout
maps the propagated representation to the task output.

\subsection{IGNN}
Following \citet{gu2020implicit}, we implement IGNN and verify it against the
official implementation (\texttt{SwiftieH/IGNN}). Its fixed point is
\[
H^\star = \phi\!\big(\widehat A^+H^\star W_{\mathrm{rec}} + B_\theta(X)\big),
\qquad
\widehat Y = r_\omega(H^\star),
\]
with $\phi=\mathrm{ReLU}$ (the canonical non-expansive activation; matches the official
code, not $\tanh$). \textbf{Well-posedness} is enforced structurally: after every
optimizer step, the columns of $W_{\mathrm{rec}}$ are projected onto an $\ell_1$ ball of
radius $\kappa/\widehat\rho_{\mathrm{PI}}(\widehat A^+)$ ($\kappa=0.9$,
$\widehat\rho_{\mathrm{PI}}(\widehat A^+)$ estimated by 20 power iterations,
$\approx 1$), guaranteeing $\phi(\widehat A^+\,\cdot\,W_{\mathrm{rec}})$
is a contraction and the fixed point unique. \textbf{Forward} solve: Picard iteration to
tolerance ($\mathrm{tol}=10^{-6}$, $\mathrm{max\_iter}=50$) under \texttt{no\_grad}.
\textbf{Backward}: exact implicit differentiation via an adjoint Picard iteration
solving $(I-J_\phi^\top)\delta=g$ for the upstream gradient $g$ --- \emph{not}
backpropagation through the forward unroll. The resulting hidden
representation is row-wise $\ell_2$-normalized before the linear readout.

\subsection{EIGNN}
\citep{liu2021eignn}, verified against the official implementation
(\texttt{liu-jc/EIGNN}). Linear fixed point with a learned positive-semidefinite
feature-propagation matrix,
\[
H^\star = \widehat A H^\star M(F) + B_\theta(X),
\qquad
M(F) = \frac{\gamma F^\top F}{\|F^\top F\|_F+\varepsilon_F},
\qquad
\widehat Y=r_\omega(H^\star),
\]
with $F\in\mathbb R^{c\times c}$ learnable, $\gamma=0.9$ (\emph{fixed, not learned} ---
matches official code), $\varepsilon_F=10^{-12}$. \textbf{Solved in closed form}, not
by iteration: with $\widehat A=Q_A\Lambda_AQ_A^\top$ (cached per graph via eigendecomposition,
$\mathcal O(n^3)$ once) and $M(F)=Q_M\Lambda_MQ_M^\top$ (recomputed every forward call,
$\mathcal O(c^3)$, cheap),
\[
[{H}_{\mathrm{spec}}]_{ij}
= \frac{[{B}_{\mathrm{spec}}]_{ij}}{1-(\Lambda_A)_{ii}(\Lambda_M)_{jj}},
\qquad
{B}_{\mathrm{spec}}=Q_A^\top B_\theta(X) Q_M,
\qquad
H^\star = Q_A{H}_{\mathrm{spec}}Q_M^\top.
\]
Verified numerically against $500$ steps of direct fixed-point iteration
($\|\cdot\|_\infty$ agreement $\sim 10^{-6}$). \textbf{Backward}: ordinary
autograd through the (differentiable) eigendecomposition of the small matrix $M(F)$;
gradients through the graph-scale eigendecomposition of $\widehat A$ are not needed since $\widehat A$ is
treated as a constant (cached, non-trainable). No implicit differentiation or
Picard unrolling is used for the reported EIGNN results.

\subsection{EnergyGNN (edge-wise)}
\citep{solodova2025graph}, reconstructed from the published paper (no official code was
released). Node embeddings are the unique minimizer of a strongly convex energy,
\[
\begin{aligned}
H^\star
&=\arg\min_H E_\theta(G,H),\\
E_\theta(G,H)
&=\sum_i\Big[u_\theta(m_i,H_i,x_i)+\tfrac{\beta_{\mathrm E}}2\|H_i\|_2^2\Big],\\
m_i
&=\sum_{j\in\mathcal N(i)}m_\theta(H_i,H_j,x_i,x_j,a_{ij})
  +m_{\mathrm{self},\theta}(H_i,x_i).
\end{aligned}
\]
with $\beta_{\mathrm E}=0.04$ (the code parameter is named
\texttt{beta}). Neighbour aggregation uses the \emph{raw}, self-loop-free
adjacency; the self-term is an independently parameterized function, not a shared
copy of $m_\theta$. Both $m_\theta$, $m_{\mathrm{self},\theta}$, and $u_\theta$ are
partially input-convex networks (PICNNs, \citealp{amos2017input}): convex and
non-decreasing in their $H$-arguments via a $\mathrm{softplus}$-nonnegative weight
path, unconstrained in $(x,a)$; edge features $a_{ij}$, when used, enter only this
unconstrained path, so convexity in $H$ is unaffected. Hidden width $4$, output width
$2$ (message) / $1$ (update), the paper's reported sizes. \textbf{Well-posedness} follows
from convexity of the PICNN composition plus the strongly-convex
$\beta_{\mathrm E}$ term.
\textbf{Forward}: L-BFGS to tolerance ($\mathrm{tol}=10^{-5}$, $\mathrm{max\_iter}=50$,
strong-Wolfe line search); convergence is \emph{not} assumed --- the gradient-norm
residual and iteration count are logged, and non-convergence is reported rather than
silently accepted. \textbf{Backward}: exact implicit differentiation via the energy
Hessian, solved with conjugate gradients using only Hessian--vector products (never
materializing $\nabla^2_{HH}E_\theta$); verified against a dense double-precision
Hessian reference to $\sim 10^{-16}$ relative error on a small graph. Because
$\mathrm{ReLU}$-based PICNNs are piecewise linear, curvature away from activation
kinks comes predominantly from the $\beta_{\mathrm E}$ term, which can slow L-BFGS convergence on
some graphs; this is a property of the architecture, not the solver implementation.

\subsection{Deep Neural Sheaf Diffusion}
\label{app:dnsd}

We use the finite-depth Deep Neural Sheaf Diffusion (DNSD) architecture of
\rev{\citet{bourgerie2026deep}}. In the notation of this paper, each node \rev{representation}
has total dimension $c=dq$, where $d$ is the stalk dimension and $q$ is the
number of feature channels per stalk. A linear encoder maps the input features
to $\mathbb R^c$, and a linear decoder maps the final \rev{representation} to the task
output.

Each of the 12 untied diffusion layers constructs separate source and target
restriction maps from the ordered endpoint \rev{representations}. For an oriented edge
$u\to v$, the two builders receive $h_u^{(\ell)}\Vert h_v^{(\ell)}$. In the
full variant, each builder is a linear map
$\mathbb R^{2c}\to\mathbb R^{d^2}$ whose output is reshaped into a
$d\times d$ matrix. In the diagonal variant, it maps
$\mathbb R^{2c}\to\mathbb R^d$ and the output forms the diagonal of the
restriction matrix. Source and target builders are distinct, and their
parameters are not shared across layers.

The resulting maps define the directional sheaf-adjacency operator used by
the released DNSD implementation (\texttt{directional=True}); the layer uses
this adjacency propagation rather than a sheaf-Laplacian residual update. The
propagated representation is passed through either $\tanh$ or ReLU. A
per-layer sigmoid gate acts on each length-$q$ channel vector within the
stalk-stacked representation (\texttt{use\_stalk\_gate=True}), followed by
LayerNorm over the flattened $c$-dimensional node \rev{representation}
(\texttt{layer\_norm\_mode=embedding}).

\rev{On the community benchmark, we evaluate the four combinations of
diagonal or full restriction maps and $\tanh$ or ReLU activation, with
$c=48$, $d=3$, and $q=16$.} All configurations use 12 layers and the common
Adam training protocol with learning rate $10^{-2}$, weight decay
$5\times10^{-4}$, at most 2000 epochs, and early-stopping patience 500.

\subsection{GraphAttnDEQ}
\label{app:graphattndeq}

GraphAttnDEQ is a mechanism-isolating ablation of SheafDEQ. It preserves the
\rev{weight-tied recurrent skeleton, fixed input drive, and additive shift, but
does not normalize the recurrent state. It} replaces the learned sheaf \rev{propagation} with
multi-head softmax graph attention recomputed from the current iterate:
\[
\rev{H^{(t+1)}
=
\sigma\!\left(\operatorname{Attn}_\theta(H^{(t)},A^+)W\right)
+f_\theta(X)+y
},
\qquad
H^{(0)}=\frac{\mathbf 1}{\sqrt c}.
\]
Self-loops are included in \(A^+\). For head \(m\),
\[
\begin{aligned}
\operatorname{Attn}_\theta(H)_i
&=\mathop{\|}_{m=1}^{M}
  \sum_{j\in\mathcal N(i)\cup\{i\}}
  \alpha_{ij}^{(m)}V^{(m)}h_j,\\
\alpha_{ij}^{(m)}
&=\operatorname{softmax}_{j}\!\left(
\frac{(Q^{(m)}h_i)^\top(K^{(m)}h_j)}{\sqrt{c/M}}
\right).
\end{aligned}
\]
The symbol \(\|\) denotes concatenation across heads. The per-head query,
key, and value projections are bias-free, and a single output matrix
\(W\in\mathbb R^{c\times c}\) is applied after concatenation. This replaces
SheafDEQ's \(W_1,W_2\) pair because the attention projections already supply
the feature-transform capacity otherwise provided inside the stalk space.

The head count is matched to the SheafDEQ stalk dimension, \(M=d\), so each
head has width \(c/M=q\). The dataset-dependent \(c\), \(K\), and \(M\) values
are given in Table~\ref{tab:model-configs}; for community detection,
\(c=48\), \(M=3\), \(c/M=16\), and \(K=20\). All reported variants use
\(\tanh\) and the same additive-shift sweep as SheafDEQ.

When edge features \(a_{ij}\) are enabled, a learned per-head linear embedding
is added to both the key and value for the corresponding edge, following
\texttt{TransformerConv}'s edge-conditioning convention
\citep{shi2021masked}. Self-loop edges receive a zero edge embedding. This
option is disabled by default and enabled only through an explicit
\texttt{edge\_dim} argument.

\paragraph{Pre-attention normalization and fixed-point convergence.}
Consider a GraphAttnDEQ variant that uses the scale-invariant row-wise
normalization principle of Section~\ref{sec:adaptive-equilibrium} only before
attention-coefficient inference. Specifically, let
\([\overline H]_{i,:}=H_{i,:}/\|H_{i,:}\|_p\), and let
\(\mathcal P_{\mathrm{att}}(H;A^+)\) denote the multi-head propagation above
with queries and keys evaluated at \(\overline H\), but values evaluated at
the original state \(H\). For every \(s>0\),
\[
\overline{sH}=\overline H
\ \Longrightarrow\
\alpha_{ij}^{(m)}(\overline{sH})=\alpha_{ij}^{(m)}(\overline H)
\ \Longrightarrow\
\mathcal P_{\mathrm{att}}(sH;A^+)=s\mathcal P_{\mathrm{att}}(H;A^+).
\]
Thus the attention propagation is positively \(1\)-homogeneous. Taking
\(\sigma=\tanh\) and \(y=\beta_{\mathrm{shift}}\), under the remaining
assumptions of Theorem~\ref{thm:sheafdeq}---the propagation is Lipschitz, the
fixed input drive satisfies \(f_\theta(X)\geq0\) entrywise, and
\(\beta_{\mathrm{shift}}\geq1.2\)---
Lemma~\ref{lem:complete-subhom} and the proof of
Theorem~\ref{thm:sheafdeq} give a unique entrywise-positive fixed point and
global linear convergence of Picard iteration from every entrywise-positive
initialization. This observation requires edge
features to affect only attention-coefficient inference, rather than introduce
additive terms into the transported values.

\rev{The GraphAttnDEQ evaluated here does not use this normalization and is
therefore not covered by this argument; no well-posedness guarantee is claimed
for its reported configurations.} It also omits the LayerNorm, feedforward block,
and residual connections used by the Loop Transformer. It is trained by
ordinary backpropagation through the \(K\)-step unroll; no implicit
differentiation is used.

\subsection{Loop Transformer}
\label{app:looptransformer}

Inspired by recurrent-depth Transformers, which repeatedly apply a shared
Transformer block~\citep{geiping2026recurrent}, we introduce a graph-masked,
finite-horizon Loop Transformer as a \rev{control} without an equilibrium
certificate. A single Transformer block---graph-masked multi-head attention
(\texttt{TransformerConv}, restricted to \(\mathcal N(i)\), with learned skip
gate \texttt{beta=True}) followed by a GELU feedforward network, each with a
residual connection and post-LN LayerNorm---is
applied with \emph{tied weights} for \(K\) iterations:
\[
H^{(0)}=B_\theta(X),
\qquad
H^{(k+1)} = \mathrm{Block}\big(H^{(k)}, A, \mathbf a_E\big),
\qquad
\widehat Y = r_\omega\big(H^{(K)}\big),
\]
with the dataset-dependent $K$ and head count in
Table~\ref{tab:model-configs} (the head count is matched to the corresponding
stalk dimension $d$ for comparability). Here
$\mathbf a_E=(a_e)_{e\in E}$ denotes the collection of edge features.
The FFN has width $2\times$hidden and dropout $p_{\mathrm{drop}}=0.1$.
Attention is recomputed from the current
$H^{(k)}$ at every step. Edge features enter natively via the attention scoring
function. When no real edge features are available, a zero dummy edge
attribute of dimension one is supplied to activate the learned skip gate on
node-feature differences. \textbf{No well-posedness guarantee of any kind} --- no contraction, spectral,
or convexity constraint is imposed, and there is no equilibrium-solve diagnostic;
this model is included specifically to test whether an expressive but theoretically
unconstrained loop is competitive with the constrained alternatives above. Trained by
ordinary backpropagation through the $K$-step unroll (parameter count independent of
$K$, verified empirically).

\subsection{Summary of solver/gradient mechanisms}
Table~\ref{tab:solver-summary} summarizes what most clearly distinguishes each
\rev{comparison model's} training mechanism from plain backpropagation through an unrolled loop.

\begin{table}[h]
\centering
\small
\begin{tabular}{lll}
\toprule
Model & Forward solve & Backward gradient \\
\midrule
MLP & feedforward (no iteration) & autograd \\
APPNP family & dataset-dependent $K$ Picard steps & unrolled autograd \\
IGNN & Picard to tolerance & implicit (adjoint Picard) \\
EIGNN & closed-form (spectral) & autograd through $\mathrm{eigh}(M(F))$ \\
EnergyGNN (edge) & L-BFGS to tolerance & implicit (CG + Hessian--vector products) \\
GraphAttnDEQ & dataset-dependent $K$ weight-tied steps & unrolled autograd \\
Loop Transformer & dataset-dependent $K$ weight-tied steps & unrolled autograd \\
\bottomrule
\end{tabular}
\caption{Forward/backward mechanism per \rev{comparison model}. IGNN, EIGNN, and EnergyGNN use
implicit or closed-form solutions specifically to avoid backpropagating through an
unrolled iterate trajectory; the APPNP family, GraphAttnDEQ, and Loop
Transformer are \rev{finite-horizon}, weight-tied \rev{models} trained by ordinary
unrolled backpropagation.}
\label{tab:solver-summary}
\end{table}

\section{Experimental Details}
\label{app:experiments}

This appendix provides the information needed to reproduce all results reported
in the paper. We describe the datasets (Section~\ref{app:datasets}), data
splits (Section~\ref{app:splits}), model architectures
(Section~\ref{app:model-configs}), model parameter counts
(Section~\ref{app:params}), training procedure
(Section~\ref{app:training}), hyperparameter selection
(Section~\ref{app:hparam}), compute infrastructure
(Section~\ref{app:compute}), and dataset-specific notes
(Sections~\ref{app:community-notes}--\ref{app:benchmark-notes}).

\subsection{Datasets}
\label{app:datasets}

\paragraph{Community detection.}
We use the synthetic community detection benchmark from
\citet{bourgerie2026deep}. Each instance is a graph with $n = 1500$ nodes
partitioned into $C = 3$ balanced communities of 500 nodes each. Node features
are sampled from two-dimensional isotropic Gaussians with standard deviation
$\sigma_{\mathrm{obs}}=3.0$ and community means $(0,0)$, $(1,0)$, and $(0,1)$. These
strongly overlapping distributions provide weak but nonzero class information.
At level $L_{\mathrm{hetero}}=0$, each community forms a separate
$k_{\mathrm{NN}}$-nearest-neighbour graph with $k_{\mathrm{NN}}=8$.
Each subsequent level rewires an additional $10\%$ of the original
within-community edges to uniformly sampled cross-community endpoints, while
preserving the edge count.

We report levels $L_{\mathrm{hetero}}\in\{1,3,5,7,9,10\}$. The level orders the amount of
cross-community mixing, from mostly within-community connectivity at low
$L_{\mathrm{hetero}}$ to fully rewired connectivity at
$L_{\mathrm{hetero}}=10$; it should not be read as a monotone
difficulty index.

The training set consists of 6 independently generated graphs (graph-generation
seeds 42--47) of 1500 nodes each; the test set consists of 3 graphs
(graph-generation seeds 100--102). Accuracy is evaluated on all nodes of all
test graphs jointly. These graph-generation seeds define the data and are
distinct from the five parameter-initialization seeds used to aggregate the
reported results.
There is no separate validation graph. Validation masks within the training
graphs are used for early stopping, as detailed in
Section~\ref{app:splits}.

\paragraph{Counting.}
A collection of 50 chain graphs of lengths 1 to 50 nodes. The regression
target at each node is the total length of its chain (an integer from 1 to
50). Degree-based node features provide local structural information, while
the model must infer the global chain length. We report negative MSE (higher
is better).

\paragraph{Sums.}
A collection of 2000 chain graphs, each of exactly 50 nodes, with binary node
features drawn uniformly from $\{0, 1\}$. The target at each node is the sum
of all binary features in its chain. The model must propagate and aggregate
information across up to 49 hops. We report negative MSE (higher is better).

\paragraph{MNIST Terrain.}
A multi-agent classification benchmark derived from MNIST. Each image is
modelled as a 10-node graph where each node represents one pixel-agent
observing a sampled pixel value and its location, with communication determined
by pixel proximity. The task is binary classification (digit 0 vs.\ digit 1).
We report test accuracy.

\paragraph{Coordinates.}
A regression benchmark with 1500 graphs of 20 nodes each. Nodes are placed
uniformly at random in $[0, 1]^2$ and connected by a
$k_{\mathrm{NN}}$-nearest-neighbour graph. Nodes receive one-hot identifiers,
and edge attributes encode pairwise Euclidean distances. The target is each
node's 2D position, evaluated through a rigid-motion-invariant edge-distance
loss; its negative is reported so that higher is better.

\subsection{Data Splits}
\label{app:splits}

Counting, Sums, and MNIST Terrain use fixed transductive node-level
training, validation, and test masks. Coordinates instead uses graph-level
splits, with test graphs held out from training. Split definitions and seeds
are shared across the five parameter-initialization runs.

All four benchmarks use a 60/20/20 train/validation/test split. The resulting
node counts are reported in Table~\ref{tab:distributed-splits}. For
Coordinates, the split is performed over graphs: 900 training, 300 validation,
and 300 test graphs, corresponding to the node counts shown in the table.

\begin{table}[t]
\centering
\small
\caption{Train, validation, and test split sizes for the distributed
benchmarks. The first three tasks use node-level masks; Coordinates uses a
graph-level split.}
\label{tab:distributed-splits}
\begin{tabular}{@{}lrrrr@{}}
\toprule
Dataset & Total nodes & Train & Validation & Test \\
\midrule
Counting      & 1{,}275 & 765    & 255    & 255 \\
Sums          & 100{,}000 & 60{,}000 & 20{,}000 & 20{,}000 \\
MNIST Terrain & 147{,}800 & 88{,}680 & 29{,}560 & 29{,}560 \\
Coordinates   & 30{,}000 & 18{,}000 & 6{,}000 & 6{,}000 \\
\bottomrule
\end{tabular}
\end{table}

All reported community-detection results use the multi-graph protocol. The
training data comprise six independently generated graphs (seeds 42--47), and
the test data comprise three separate graphs (seeds 100--102). All nodes
participate in the graph structure and forward computation. An 80/20 split of
each training graph's nodes determines the cross-entropy training loss and the
validation metric used for early stopping, respectively. Test accuracy is
evaluated on the three separate test graphs, which are never used during
training or validation.

\subsection{Model Configurations}
\label{app:model-configs}

\rev{DNSD is evaluated on the community benchmark using the 12-layer
configurations documented in Appendix~\ref{app:dnsd}.} The other compared
architectures use the total \rev{representation} dimension $c$ and iteration
budget $K$ in Table~\ref{tab:model-configs} within each dataset. For SheafDEQ, the
\rev{representation} is organized as $c=dq$, where $d$ is the stalk dimension and $q$
is the number of sheaf feature channels.

\begin{table}[h]
\centering
\small
\caption{Per-dataset configuration used by the recurrent comparison models;
\rev{DNSD uses separate 12-layer configurations for the community benchmark
(Appendix~\ref{app:dnsd}).}
$c$: total node \rev{representation} dimension. $K$: iteration/layer budget. $d$: stalk
dimension for SheafDEQ; $d$ also sets the GraphAttnDEQ and Loop Transformer
head counts.
$q=c/d$: number of sheaf feature channels.}
\label{tab:model-configs}
\begin{tabular}{@{}lrrrrr@{}}
\toprule
Dataset & $c$ & $K$ & $d$ & $q$ & Metric \\
\midrule
Counting            & 6  & 50  & 3 & 2  & MSE   \\
Sums                & 6  & 50  & 3 & 2  & MSE   \\
MNIST Terrain       & 16 & 10  & 2 & 8  & Acc.  \\
Coordinates         & 16 & 10  & 2 & 8  & MSE   \\
Community (all $L_{\mathrm{hetero}}$) & 48 & 20  & 3 & 16 & Acc.  \\
\bottomrule
\end{tabular}
\end{table}

\paragraph{SheafDEQ.}
The fixed drive $B_\theta(X)$ is a single linear layer followed by ReLU,
mapping node features to the total \rev{representation} dimension:
$B_\theta:\mathbb R^f\to\mathbb R^c$. For sheaf propagation, its output is
reshaped from $n\times c$ into $nd\times q$.

The restriction-map generators $g_{\psi_{\mathrm{src}}}$ and
$g_{\psi_{\mathrm{tgt}}}$ are separate two-layer MLPs, one per incidence
direction, and both are weight-tied across iterations. Each MLP takes the concatenation
of two endpoint representations (each of total dimension $c$) plus edge features
when present, and outputs a $d\times d$ restriction matrix:
\[
g_{\psi_{\mathrm{src}}},g_{\psi_{\mathrm{tgt}}}
:\mathbb R^{2c+d_e}\to\mathbb R^{d\times d},
\]
each with architecture $[2c+d_e]\to c\to d^2$ (ReLU hidden activation).
The output is normalized by its Frobenius norm (\texttt{frobenius\_norm=True}).

The Picard iteration uses $W_1\in\mathbb R^{d\times d}$ to mix stalk
coordinates and $W_2\in\mathbb R^{q\times q}$ to mix sheaf feature channels.
The implementation matches the manuscript convention
\[
\mathcal P_\Theta(H;G)
=\mathscr R^{-1}\!\left(
\mathcal P_{\mathcal F_\psi(H;G)}
(I_n\otimes W_1)\mathscr R(H)W_2
\right).
\]
Thus $W_1$ acts blockwise on the $d$ stalk coordinates through the
Kronecker factor, whereas $W_2$ acts on the $q$ feature channels.
The propagation output is then passed through the shifted activation
$\sigma_{\beta_{\mathrm{shift}}}(t)=\tanh(t)+\beta_{\mathrm{shift}}$ and the
fixed drive $B_\theta(X)$. No output normalization is used.

The predictor is a single linear readout
$r_\omega:\mathbb R^c\to\mathbb R^o$, where $o=C$ for classification,
$o=1$ for scalar regression, and $o=2$ for Coordinates.

Two variants are evaluated:
\begin{itemize}
    \item \textbf{pre-norm=True}: after reshaping to $nd\times q$, each
    length-$q$ row is $\ell_2$-normalized independently across its feature
    channels before its node
    block is passed to both restriction-map generators
    (\texttt{norm\_before\_restriction=True}). This enforces scale invariance
    and is the variant covered by the main theorem.
    \item \textbf{pre-norm=False}: representations are passed directly to the
    restriction-map generators
    without normalization (\texttt{norm\_before\_restriction=False}). This
    ablation breaks the scale-invariance condition.
\end{itemize}

\paragraph{\rev{Comparison-model configurations}.}
The equations, architectural components, solver mechanisms, and gradient
computations for IGNN, EIGNN, EnergyGNN, the APPNP family, GraphAttnDEQ, the
Loop Transformer, and the feature-only MLP are documented once in
Appendix~\ref{app:baselines}. These models use the dataset-dependent dimensions
and computation budgets in Table~\ref{tab:model-configs}; their fixed
hyperparameters and validation-based configuration-selection rule are given
in Appendix~\ref{app:hparam}. \rev{For the community benchmark, DNSD uses
the 12-layer configurations in Appendix~\ref{app:dnsd}.} Benchmark-specific coverage and solver failures are
reported in Appendix~\ref{app:benchmark-notes}.

\subsection{Parameter Counts}
\label{app:params}

Table~\ref{tab:params} reports audited trainable-parameter counts for the
community configuration. The counts are obtained from the current
implementation using \texttt{sum(p.numel() for p in model.parameters())}.
APPNP shift variants have the same count as APPNP because the shift is fixed,
and the two SheafDEQ \rev{normalization} variants likewise have identical
counts.

SheafDEQ's parameters are weight-tied across all $K$ iterations. The counts
reported therefore do not grow with the number of equilibrium iterations. The same
weight-tying principle applies to GraphAttnDEQ and the Loop Transformer: each
recurrent operator is instantiated once and reused for all $K$ steps. DNSD,
in contrast, uses 12 untied layers. Its $\tanh$ and ReLU variants have the
same parameter count for a fixed diagonal or full restriction-map type. The
depth-parity MLP has comparatively many parameters because its $K-2$ hidden
$c\times c$ maps are not tied.

\begin{table*}[h]
\centering
\small
\caption{Audited trainable-parameter counts from the current implementation
for the community configuration. Experimental checkpoints may differ
by at most one $c$-dimensional bias vector (at most 48 parameters) because of
a minor code revision; this bookkeeping difference does not affect any
reported result. APPNP shift variants, DNSD activation variants of a fixed map
type, and the two SheafDEQ \rev{normalization} variants have identical counts.}
\label{tab:params}
\begin{tabular}{@{}lr@{}}
\toprule
Model & Community \\
\midrule
MLP (depth parity)       & 42,675 \\
APPNP                    & 2,691 \\
IGNN                     & 2,643 \\
EIGNN                    & 2,643 \\
EnergyGNN                & 340 \\
GraphAttnDEQ             & 9,555 \\
Loop Transformer         & 21,843 \\
DNSD diagonal            & 11,691 \\
DNSD full                & 25,659 \\
\textsc{SheafDEQ}        & 11,045 \\
\bottomrule
\end{tabular}
\end{table*}

\subsection{Training Procedure}
\label{app:training}

All models are trained with Adam \citep{kingma2015adam}, learning rate
$\eta_{\mathrm{lr}} = 10^{-2}$, and weight decay
$\lambda_{\mathrm{wd}} = 5 \times 10^{-4}$. Training
runs for at most 2000 epochs with early stopping based on the validation
metric (accuracy for classification, negative MSE for regression).

The patience is 200 epochs for the Counting, Sums, and Coordinates
benchmarks, and 500 epochs for MNIST Terrain and all community detection
levels, reflecting their longer convergence times. Early stopping is applied
per seed independently; the best checkpoint (highest validation metric across
all epochs seen) is used for test evaluation.

\paragraph{Gradient computation.}
For SheafDEQ, GraphAttnDEQ, the Loop Transformer, and the APPNP family, gradients are
computed by unrolled backpropagation through $K$ iterations (BPTT). No
implicit differentiation is used. In the predictive experiments,
SheafDEQ therefore returns the $K$th iterate rather than running a
residual-based solver to tolerance; these tables measure predictive performance
at the common iteration budget, not numerical convergence to equilibrium.
For IGNN, the forward pass finds $H^\star$ by Picard iteration and the backward
pass uses implicit differentiation through the fixed-point equation. EIGNN
instead uses its closed-form spectral solution and ordinary autograd through
the eigendecomposition of the learnable feature-side matrix $M(F)$; the cached
graph eigendecomposition is constant, and neither Picard unrolling nor IFT is
used. For EnergyGNN, the backward pass uses IFT through the L-BFGS minimizer.

For IGNN, well-posedness is maintained by projecting the recurrent weight $W$
after every gradient step (\texttt{project\_weights()}).

\paragraph{\rev{Parameter and recurrent-state initialization}.}
All trainable parameters are initialized with Xavier uniform initialization
\citep{glorot2010understanding}. For SheafDEQ and GraphAttnDEQ, the recurrent
\rev{state} is initialized as a unit vector of ones (normalized row-wise to unit
$\ell_2$ norm) regardless of the training seed. The training seed only affects
the model weight initialization.

\paragraph{Reported metrics.}
\rev{Each entry reports the test median $[\min,\max]$ over five independent
parameter-initialization seeds (42, 43, 44, 45, 46). For each seed, test
performance is evaluated at the checkpoint with the best validation metric.} For community detection, these parameter seeds are
separate from the graph-generation seeds used to construct the fixed training
and test graph sets.

\rev{For each model family, we select the configuration with the highest
mean validation performance across five seeds and report its test median
$[\min,\max]$. The full appendix tables list every configuration separately.}

\subsection{Hyperparameter Selection}
\label{app:hparam}

The only hyperparameter swept for SheafDEQ and GraphAttnDEQ is
$\beta_{\mathrm{shift}}\in\{1.2,
1.3, 1.4, 1.6\}$. All other architectural and training parameters are fixed
across datasets within each dataset's model configuration. The best
$\beta_{\mathrm{shift}}$ is defined to be the configuration with the highest mean
validation metric over the 5 parameter-initialization seeds, separately for
each dataset and, for SheafDEQ, each \rev{normalization} variant.
The same validation-based rule is used to select among the explicitly
evaluated APPNP configurations \rev{and, on the community benchmark, DNSD
configurations}.

Apart from these explicitly evaluated configuration families, \rev{comparison-model}
hyperparameters are fixed to the defaults from their respective papers.
Specifically:
IGNN: \texttt{kappa=0.9}, \texttt{max\_iter=50}, \texttt{fp\_tol=1e-6};
EIGNN: \texttt{gamma=0.9};
EnergyGNN: \texttt{beta=0.04}, \texttt{solver\_tol=1e-5}, \texttt{solver\_max\_iter=50};
APPNP: \texttt{alpha=0.1}, \texttt{dropout=0.5};
APPNP-$\beta_{\mathrm{shift}}$: \texttt{alpha=0.1}, \texttt{dropout=0.5}, shift as noted;
GraphAttnDEQ: \(M=d\), \(\tanh\) activation, and \rev{no recurrent-state normalization,} LayerNorm, FFN, or residual connection;
Loop Transformer: \texttt{ffn\_mult=2}, \texttt{dropout=0.1}.

No hyperparameter search was performed for any \rev{comparison model} beyond the above.

\subsection{Compute Infrastructure}
\label{app:compute}

All benchmark experiments were run on the Dardel HPC cluster at the
PDC Center for High Performance Computing\footnote{\url{https://www.pdc.kth.se/}},
using AMD MI250X GPUs. Model/seed combinations were run in parallel. The total
usage across all datasets, models, seeds, and $\beta_{\mathrm{shift}}$
configurations was approximately 500 GPU-hours.

Robustness experiments (\rev{initial-state sensitivity} and continued iteration) were run on a
local CPU machine. Each training run for the robustness experiments used
\texttt{patience=500} and at most 2000 epochs.

For the continued-iteration and bounded-staleness diagnostics, we select the
lowest-index checkpoint for each \rev{evaluated model} achieving at least $80\%$
validation accuracy at its trained inference horizon on
$L_{\mathrm{hetero}}=7$. This criterion selects seed 43 for both SheafDEQ
variants and seed 42 for the Loop Transformer.

Models were implemented in PyTorch with PyTorch Geometric for graph data
handling. All model definitions use a shared training routine.

\subsection{Community Detection: Additional Notes}
\label{app:community-notes}

\begin{table*}[t]
\centering
\caption{Community detection: test accuracy (\%) across rewiring levels L1--L10. \rev{Each entry reports the median $[\min,\max]$ over five parameter-initialization seeds.} GraphAttnDEQ rows show all $\beta_{\mathrm{shift}}$ configurations. \textsc{SheafDEQ} rows show the same configurations ($\beta_{\mathrm{shift}} \in \{1.2, 1.3, 1.4, 1.6\}$) for both the pre-restriction-normalized variant (covered by the conditional contraction guarantee) and its ablation. \rev{Wide ranges} at intermediate levels indicate substantial variation across parameter-initialization seeds, but the aggregate statistics do not distinguish optimization variability from \rev{differences in the finite-iteration states reached at inference}. \rone{1st}, \rtwo{2nd}, \rthree{3rd} best per column.}
\label{tab:community-full}
\resizebox{\textwidth}{!}{%
\begin{tabular}{@{}l l c c c c c c@{}}
\toprule
\multicolumn{2}{l}{\textbf{Model}} & \textbf{L1} & \textbf{L3} & \textbf{L5} & \textbf{L7} & \textbf{L9} & \textbf{L10} \\
\midrule
  \multicolumn{2}{l}{MLP} & 33.6\,\scriptsize{[33.3,34.8]} & 33.3\,\scriptsize{[33.3,37.4]} & 33.3\,\scriptsize{[32.8,33.4]} & 33.3\,\scriptsize{[31.5,33.5]} & 33.3\,\scriptsize{[33.3,34.0]} & 33.3\,\scriptsize{[33.3,36.3]} \\
  \multicolumn{2}{l}{APPNP} & 39.4\,\scriptsize{[39.3,39.8]} & 38.6\,\scriptsize{[38.5,39.1]} & 37.8\,\scriptsize{[36.3,38.9]} & 37.9\,\scriptsize{[37.7,37.9]} & 38.4\,\scriptsize{[37.1,38.7]} & 38.0\,\scriptsize{[37.8,38.4]} \\
  \multicolumn{2}{l}{APPNP-$\beta_{\mathrm{shift}}{=}1.25$} & 40.0\,\scriptsize{[39.4,40.5]} & 40.8\,\scriptsize{[39.4,40.8]} & 40.7\,\scriptsize{[40.6,41.0]} & 40.1\,\scriptsize{[39.8,40.7]} & 40.2\,\scriptsize{[39.8,40.7]} & 40.0\,\scriptsize{[39.5,40.2]} \\
  \multicolumn{2}{l}{APPNP-$\beta_{\mathrm{shift}}{=}1.65$} & 40.6\,\scriptsize{[40.5,40.9]} & 40.8\,\scriptsize{[40.6,40.9]} & 40.3\,\scriptsize{[39.0,40.8]} & 40.3\,\scriptsize{[40.0,40.8]} & 40.4\,\scriptsize{[40.3,40.7]} & 40.1\,\scriptsize{[39.3,40.8]} \\
  \multicolumn{2}{l}{IGNN} & 42.2\,\scriptsize{[40.9,42.8]} & 43.6\,\scriptsize{[43.0,44.6]} & 44.6\,\scriptsize{[43.9,45.0]} & 45.0\,\scriptsize{[44.8,45.4]} & 45.8\,\scriptsize{[45.5,46.7]} & 52.7\,\scriptsize{[51.9,54.2]} \\
  \multicolumn{2}{l}{Loop Transformer} & 40.4\,\scriptsize{[40.0,43.4]} & 41.0\,\scriptsize{[38.7,42.9]} & 41.8\,\scriptsize{[40.7,45.5]} & 98.8\,\scriptsize{[97.0,99.2]} & 90.5\,\scriptsize{[86.2,92.6]} & \rone{99.1\,\scriptsize{[98.8,99.1]}} \\
\midrule
  \multirow{4}{*}{\textsc{GraphAttnDEQ}} & $\beta_{\mathrm{shift}}=1.2$ & 42.4\,\scriptsize{[37.4,45.0]} & 46.8\,\scriptsize{[43.5,59.9]} & 44.4\,\scriptsize{[41.5,54.3]} & 41.3\,\scriptsize{[39.0,59.3]} & 72.3\,\scriptsize{[54.3,81.9]} & 68.7\,\scriptsize{[64.1,97.1]} \\
   & $\beta_{\mathrm{shift}}=1.3$ & 43.1\,\scriptsize{[42.1,45.7]} & 43.9\,\scriptsize{[42.2,59.8]} & 42.6\,\scriptsize{[40.9,45.2]} & 43.8\,\scriptsize{[42.1,44.5]} & 68.6\,\scriptsize{[53.3,94.1]} & 72.0\,\scriptsize{[68.9,98.3]} \\
   & $\beta_{\mathrm{shift}}=1.4$ & 42.6\,\scriptsize{[42.0,44.2]} & 44.2\,\scriptsize{[42.8,47.7]} & 43.2\,\scriptsize{[42.4,55.1]} & 43.9\,\scriptsize{[43.0,45.3]} & 64.7\,\scriptsize{[61.6,70.6]} & 64.5\,\scriptsize{[50.6,67.5]} \\
   & $\beta_{\mathrm{shift}}=1.6$ & 41.8\,\scriptsize{[40.9,43.0]} & 43.2\,\scriptsize{[41.7,51.1]} & 42.7\,\scriptsize{[42.1,44.2]} & 45.5\,\scriptsize{[41.8,46.1]} & 63.9\,\scriptsize{[61.3,70.0]} & 68.6\,\scriptsize{[55.9,96.8]} \\
\midrule
  \multirow{4}{*}{\textsc{SheafDEQ}} & pre-norm=\!\checkmark, $\beta_{\mathrm{shift}}=1.2$ & 42.1\,\scriptsize{[41.6,44.4]} & 59.8\,\scriptsize{[45.7,96.5]} & 70.6\,\scriptsize{[56.4,93.8]} & 98.9\,\scriptsize{[98.6,99.2]} & \rthree{97.7\,\scriptsize{[96.2,98.4]}} & 98.6\,\scriptsize{[96.9,99.1]} \\
   & pre-norm=\!\checkmark, $\beta_{\mathrm{shift}}=1.3$ & 42.9\,\scriptsize{[41.5,55.6]} & 69.5\,\scriptsize{[61.6,99.3]} & 91.6\,\scriptsize{[51.3,99.4]} & \rone{99.0\,\scriptsize{[94.8,99.1]}} & \rtwo{98.0\,\scriptsize{[95.6,98.3]}} & 98.7\,\scriptsize{[97.9,98.9]} \\
   & pre-norm=\!\checkmark, $\beta_{\mathrm{shift}}=1.4$ & 42.0\,\scriptsize{[40.8,44.3]} & 56.1\,\scriptsize{[50.8,99.4]} & \rone{98.9\,\scriptsize{[56.2,99.6]}} & \rone{99.0\,\scriptsize{[98.5,99.2]}} & \rone{98.6\,\scriptsize{[94.1,98.6]}} & \rthree{98.8\,\scriptsize{[98.4,98.9]}} \\
   & pre-norm=\!\checkmark, $\beta_{\mathrm{shift}}=1.6$ & 41.4\,\scriptsize{[41.0,47.3]} & 53.1\,\scriptsize{[47.5,72.1]} & \rthree{97.7\,\scriptsize{[63.0,99.9]}} & 90.9\,\scriptsize{[49.2,99.2]} & 97.6\,\scriptsize{[95.4,98.4]} & 98.2\,\scriptsize{[97.7,98.8]} \\
\cmidrule{2-8}
  \multirow{4}{*}{\shortstack{\textsc{SheafDEQ}\\\rev{(ablation)}}} & pre-norm=\!$\times$, $\beta_{\mathrm{shift}}=1.2$ & 41.1\,\scriptsize{[40.5,63.4]} & 68.0\,\scriptsize{[46.6,98.8]} & \rtwo{98.3\,\scriptsize{[77.9,99.7]}} & 98.5\,\scriptsize{[71.3,99.0]} & 97.6\,\scriptsize{[95.1,98.3]} & 98.5\,\scriptsize{[97.8,99.2]} \\
   & pre-norm=\!$\times$, $\beta_{\mathrm{shift}}=1.3$ & 44.1\,\scriptsize{[40.8,47.3]} & 61.7\,\scriptsize{[44.1,98.0]} & 66.3\,\scriptsize{[54.4,99.6]} & \rone{99.0\,\scriptsize{[73.1,99.4]}} & \rtwo{98.0\,\scriptsize{[97.5,99.0]}} & \rtwo{98.8\,\scriptsize{[98.6,99.1]}} \\
   & pre-norm=\!$\times$, $\beta_{\mathrm{shift}}=1.4$ & 41.6\,\scriptsize{[41.1,57.0]} & 51.0\,\scriptsize{[47.7,64.0]} & 56.3\,\scriptsize{[46.9,99.6]} & 98.9\,\scriptsize{[61.7,99.5]} & 97.6\,\scriptsize{[97.0,98.4]} & 98.4\,\scriptsize{[97.8,99.0]} \\
   & pre-norm=\!$\times$, $\beta_{\mathrm{shift}}=1.6$ & 42.2\,\scriptsize{[40.3,57.0]} & 61.0\,\scriptsize{[55.2,93.0]} & 70.6\,\scriptsize{[47.4,99.8]} & 98.9\,\scriptsize{[71.7,99.3]} & 97.5\,\scriptsize{[94.5,99.2]} & \rtwo{98.8\,\scriptsize{[96.8,99.1]}} \\
\midrule
  \multicolumn{2}{l}{DNSD-diag (tanh)} & \rone{71.4\,\scriptsize{[42.1,80.3]}} & \rtwo{79.3\,\scriptsize{[74.4,90.0]}} & 86.3\,\scriptsize{[80.9,90.1]} & 86.9\,\scriptsize{[80.8,88.4]} & 61.2\,\scriptsize{[55.6,70.2]} & 96.7\,\scriptsize{[95.2,97.5]} \\
  \multicolumn{2}{l}{DNSD-diag (relu)} & 49.5\,\scriptsize{[42.4,71.0]} & \rone{87.5\,\scriptsize{[57.1,93.9]}} & 78.8\,\scriptsize{[70.2,85.6]} & 86.9\,\scriptsize{[64.7,92.7]} & 63.5\,\scriptsize{[50.9,72.7]} & \rthree{97.4\,\scriptsize{[94.0,98.0]}} \\
  \multicolumn{2}{l}{DNSD-full (tanh)} & \rthree{54.4\,\scriptsize{[42.3,70.3]}} & 70.4\,\scriptsize{[65.4,86.7]} & 86.3\,\scriptsize{[82.1,89.1]} & 84.8\,\scriptsize{[83.5,90.4]} & 57.8\,\scriptsize{[55.4,59.6]} & 95.6\,\scriptsize{[93.9,97.9]} \\
  \multicolumn{2}{l}{DNSD-full (relu)} & \rtwo{56.6\,\scriptsize{[45.7,65.4]}} & \rthree{75.5\,\scriptsize{[70.1,91.3]}} & 90.4\,\scriptsize{[84.0,93.2]} & 80.6\,\scriptsize{[75.1,93.2]} & 59.6\,\scriptsize{[54.8,67.9]} & 95.9\,\scriptsize{[95.1,97.5]} \\
\bottomrule
\end{tabular}%
}
\end{table*}

\begin{figure*}[t]
\centering
\includegraphics[width=\linewidth]{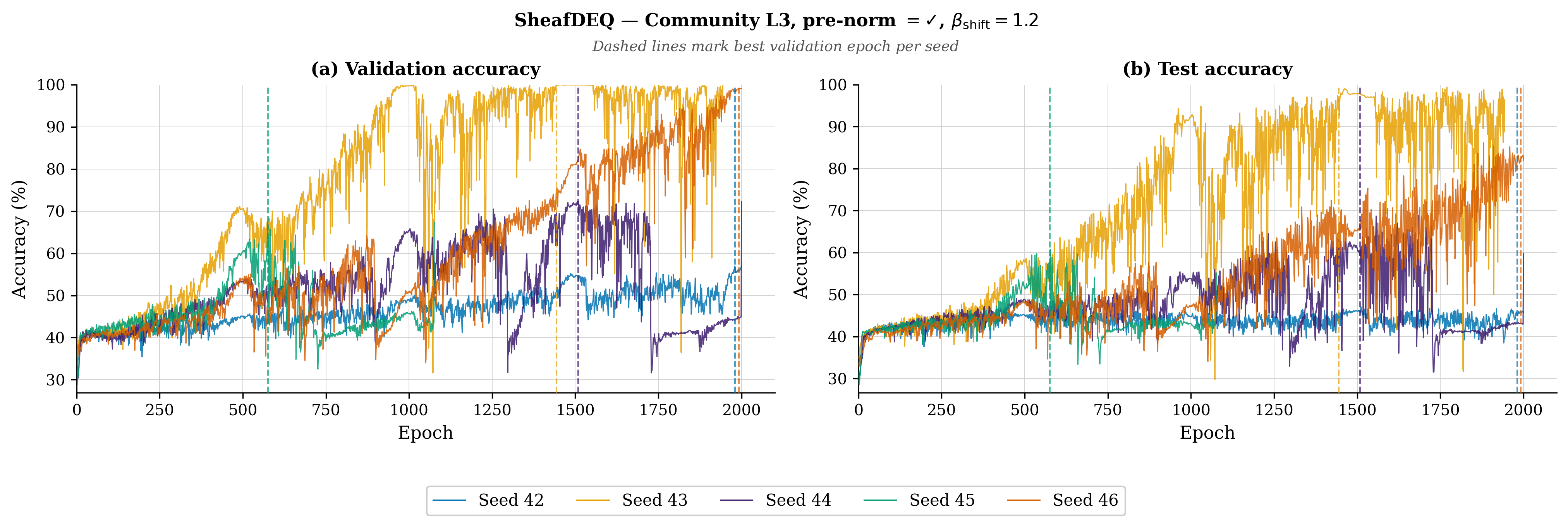}
\caption{\rev{\textbf{Seed-dependent training trajectories.} Validation and
test accuracy for pre-restriction-normalized SheafDEQ at
$L_{\mathrm{hetero}}=3$ with
$\beta_{\mathrm{shift}}=1.2$. Colours denote parameter-initialization seeds;
dashed lines mark the best validation epoch for each seed.}}
\label{fig:community-training-variability}
\end{figure*}

\paragraph{Training variability.}
\rev{Figure~\ref{fig:community-training-variability} shows that, already at the
relatively low heterophily level $L_{\mathrm{hetero}}=3$, unrolled training
follows substantially different trajectories across parameter
initializations. This is consistent with the documented sensitivity of DEQ
training to initialization and fixed-point stability
\citep{agarwala2022deep,bai2021stabilizing}. Proposed remedies include
structured orthogonal or symmetric initialization
\citep{agarwala2022deep} and Jacobian regularization
\citep{bai2021stabilizing}; we leave their evaluation for SheafDEQ to future
work. One possible caveat in our setup
is the fixed learning rate $10^{-2}$. We did not evaluate smaller learning
rates or gradual decay schedules because of the additional training cost; for
comparison, \citet{solodova2025graph} use a learning rate of $2\times10^{-3}$,
exponential decay every 200 epochs, and train for up to 5000 epochs.}

\paragraph{Excluded baselines.}
EIGNN and EnergyGNN are excluded from all community detection levels. EIGNN
would decompose the complete 9{,}000-node disjoint-union training graph
(six graphs of 1{,}500 nodes). Although this is below the configured
\texttt{max\_eig\_nodes=10000} safety threshold, the dense $O(n^2)$ spectral
storage and transformations were considered impractical for these
experiments. EnergyGNN's L-BFGS solver did not converge on these graphs within
the \rev{solver} budget.

\paragraph{Results interpretation.}
For each model family with multiple configurations in compact
Table~\ref{tab:community}, each column reports the configuration with the
highest mean validation accuracy across the five training seeds and
\rev{reports that configuration's test median and minimum--maximum range}. Different columns
may therefore select different configurations. Table~\ref{tab:community-full}
reports the fixed-configuration results needed to compare configurations
directly.
Neither predictive table measures the corresponding contraction rate or
convergence speed.

\subsection{\rev{Distributed-task Results: Additional Notes}}
\label{app:benchmark-notes}

\begin{table*}[t]
\centering
\small
\caption{\textbf{Complete distributed benchmark results.} Each entry reports
\rev{the median $[\min,\max]$ over five independent parameter-initialization seeds.}
MNIST reports accuracy (\%); Counting and Sums report $-\mathrm{MSE}$;
Coordinates reports $-\mathcal{L}_{\mathrm{dist}}$ (all higher is better).
GraphAttnDEQ rows show all
$\beta_{\mathrm{shift}}\in\{1.2,1.3,1.4,1.6\}$ configurations. \textsc{SheafDEQ}
rows show the same configurations for both \rev{normalization} variants.
\rev{The \(\dagger\) denotes that the EnergyGNN optimization solver did not
converge for any seed; no predictive score is reported. Other ``---'' entries
denote unavailable evaluations.} \rone{1st}, \rtwo{2nd}, \rthree{3rd} best
per column; \rev{ties at the displayed precision share a rank}.}
\label{tab:solodova-full}
\resizebox{\textwidth}{!}{%
\begin{tabular}{@{}l l c c c c@{}}
\toprule
\multicolumn{2}{l}{\textbf{Model}} & \textbf{Counting} & \textbf{Sums} & \textbf{MNIST} & \textbf{Coordinates} \\
\multicolumn{2}{l}{} & \scriptsize{$-\mathrm{MSE}\ \uparrow$} & \scriptsize{$-\mathrm{MSE}\ \uparrow$} & \scriptsize{\textit{Acc.\,$\uparrow$}} & \scriptsize{$-\mathcal{L}_{\mathrm{dist}}\ \uparrow$} \\
\midrule
  \multicolumn{2}{l}{MLP} & -0.367\,\scriptsize{[-0.574,-0.347]} & -0.154\,\scriptsize{[-0.999,-0.145]} & 53.3\,\scriptsize{[53.3,53.3]} & -0.105\,\scriptsize{[-0.106,-0.100]} \\
  \multicolumn{2}{l}{APPNP} & -0.335\,\scriptsize{[-0.348,-0.309]} & -0.136\,\scriptsize{[-0.140,-0.134]} & 74.8\,\scriptsize{[74.3,75.3]} & -0.030\,\scriptsize{[-0.032,-0.029]} \\
  \multicolumn{2}{l}{APPNP-$\beta_{\mathrm{shift}}{=}1.25$} & -0.342\,\scriptsize{[-0.360,-0.318]} & -0.135\,\scriptsize{[-0.142,-0.134]} & 53.3\,\scriptsize{[53.3,53.3]} & -0.044\,\scriptsize{[-0.044,-0.043]} \\
  \multicolumn{2}{l}{APPNP-$\beta_{\mathrm{shift}}{=}1.65$} & -0.337\,\scriptsize{[-0.358,-0.320]} & -0.139\,\scriptsize{[-0.142,-0.138]} & 61.3\,\scriptsize{[60.8,62.2]} & -0.051\,\scriptsize{[-0.064,-0.041]} \\
  \multicolumn{2}{l}{IGNN} & -0.336\,\scriptsize{[-0.681,-0.307]} & -0.130\,\scriptsize{[-0.998,-0.128]} & 76.1\,\scriptsize{[75.6,79.9]} & -0.028\,\scriptsize{[-0.028,-0.027]} \\
  \multicolumn{2}{l}{EIGNN} & \rev{\rone{-0.309\,\scriptsize{[-0.329,-0.280]}}} & --- & --- & --- \\
  \multicolumn{2}{l}{EnergyGNN} & -0.705\,\scriptsize{[-0.895,-0.305]} & -0.184\,\scriptsize{[-0.589,-0.140]} & 60.4\,\scriptsize{[53.3,68.2]} & \rev{---\({}^{\dagger}\)} \\
  \multicolumn{2}{l}{Loop Transformer} & -0.335\,\scriptsize{[-0.361,-0.280]} & -0.140\,\scriptsize{[-0.143,-0.139]} & \rev{\rone{95.7\,\scriptsize{[95.3,95.9]}}} & \rone{-0.006\,\scriptsize{[-0.009,-0.006]}} \\
\midrule
  \multirow{4}{*}{\textsc{GraphAttnDEQ}} & $\beta_{\mathrm{shift}}=1.2$ & -0.323\,\scriptsize{[-0.331,-0.263]} & -0.136\,\scriptsize{[-0.138,-0.129]} & \rev{\rtwo{91.8\,\scriptsize{[90.2,91.8]}}} & -0.029\,\scriptsize{[-0.029,-0.028]} \\
   & $\beta_{\mathrm{shift}}=1.3$ & -0.323\,\scriptsize{[-0.340,-0.293]} & \rev{\rthree{-0.129\,\scriptsize{[-0.138,-0.127]}}} & 91.2\,\scriptsize{[90.2,92.1]} & -0.029\,\scriptsize{[-0.029,-0.028]} \\
   & $\beta_{\mathrm{shift}}=1.4$ & \rev{\rtwo{-0.316\,\scriptsize{[-0.341,-0.283]}}} & -0.131\,\scriptsize{[-0.138,-0.127]} & \rev{\rthree{91.4\,\scriptsize{[90.9,92.3]}}} & -0.028\,\scriptsize{[-0.029,-0.012]} \\
   & $\beta_{\mathrm{shift}}=1.6$ & -0.322\,\scriptsize{[-0.341,-0.297]} & \rev{\rtwo{-0.128\,\scriptsize{[-0.138,-0.120]}}} & \rev{\rthree{91.4\,\scriptsize{[90.0,91.5]}}} & -0.028\,\scriptsize{[-0.029,-0.013]} \\
\midrule
  \multirow{4}{*}{\textsc{SheafDEQ}} & pre-norm=\!\checkmark, $\beta_{\mathrm{shift}}=1.2$ & -0.328\,\scriptsize{[-0.353,-0.314]} & \rev{\rtwo{-0.128\,\scriptsize{[-0.142,-0.125]}}} & 89.9\,\scriptsize{[87.8,90.6]} & \rtwo{-0.009\,\scriptsize{[-0.014,-0.007]}} \\
   & pre-norm=\!\checkmark, $\beta_{\mathrm{shift}}=1.3$ & -0.325\,\scriptsize{[-0.359,-0.319]} & \rev{\rtwo{-0.128\,\scriptsize{[-0.142,-0.121]}}} & 89.0\,\scriptsize{[85.4,92.1]} & -0.012\,\scriptsize{[-0.018,-0.008]} \\
   & pre-norm=\!\checkmark, $\beta_{\mathrm{shift}}=1.4$ & \rev{\rthree{-0.317\,\scriptsize{[-0.359,-0.304]}}} & \rev{\rtwo{-0.128\,\scriptsize{[-0.138,-0.125]}}} & 88.7\,\scriptsize{[84.5,90.6]} & -0.012\,\scriptsize{[-0.018,-0.007]} \\
   & pre-norm=\!\checkmark, $\beta_{\mathrm{shift}}=1.6$ & -0.337\,\scriptsize{[-0.355,-0.314]} & \rev{\rone{-0.127\,\scriptsize{[-0.139,-0.124]}}} & 89.0\,\scriptsize{[88.2,90.3]} & \rthree{-0.011\,\scriptsize{[-0.013,-0.007]}} \\
  \cmidrule{2-6}
  \multirow{4}{*}{\shortstack{\textsc{SheafDEQ}\\\rev{(ablation)}}} & pre-norm=\!$\times$, $\beta_{\mathrm{shift}}=1.2$ & -0.319\,\scriptsize{[-0.355,-0.314]} & \rev{\rtwo{-0.128\,\scriptsize{[-0.133,-0.122]}}} & 91.1\,\scriptsize{[87.4,93.4]} & -0.020\,\scriptsize{[-0.021,-0.007]} \\
   & pre-norm=\!$\times$, $\beta_{\mathrm{shift}}=1.3$ & -0.327\,\scriptsize{[-0.355,-0.314]} & -0.130\,\scriptsize{[-0.135,-0.123]} & 86.8\,\scriptsize{[83.9,89.9]} & -0.014\,\scriptsize{[-0.019,-0.011]} \\
   & pre-norm=\!$\times$, $\beta_{\mathrm{shift}}=1.4$ & -0.327\,\scriptsize{[-0.356,-0.314]} & \rev{\rthree{-0.129\,\scriptsize{[-0.134,-0.110]}}} & 89.9\,\scriptsize{[89.1,93.7]} & -0.020\,\scriptsize{[-0.020,-0.013]} \\
   & pre-norm=\!$\times$, $\beta_{\mathrm{shift}}=1.6$ & -0.326\,\scriptsize{[-0.351,-0.312]} & \rev{\rtwo{-0.128\,\scriptsize{[-0.134,-0.125]}}} & 88.4\,\scriptsize{[86.2,90.8]} & -0.017\,\scriptsize{[-0.029,-0.010]} \\
\bottomrule
\end{tabular}%
}
\end{table*}

\paragraph{\rev{Counting.}}
\rev{The similar results for SheafDEQ and the attention ablation suggest that
adaptive edge transformations do not resolve the main difficulty: one
recurrent update must preserve and transmit a global count through many
applications.}

\paragraph{\rev{Sums.}}
\rev{Every edge supports the same operation: accumulating binary values along
a chain. Scalar attention and matrix-valued sheaf transport therefore offer
similar useful computations. Their close results suggest little need for
different transformations on different edges.}

\paragraph{\rev{MNIST Terrain.}}
\rev{The gains over fixed-propagation baselines support adapting how pixel
observations are combined. The attention and sheaf models remain close,
while the finite-horizon control performs better. This task therefore
supports adaptive recurrent computation but does not show a specific
advantage from matrix-valued transport.}

\paragraph{\rev{Coordinates.}}
\rev{Messages arriving through different edges encode different geometric
constraints. Matrix-valued restriction maps can transform these messages
into edge-dependent feature coordinates before aggregation. This matches
the clear separation between SheafDEQ and its attention ablation, making
Coordinates the strongest distributed evidence for the sheaf component.}

\paragraph{Coordinates solver failure.}
\rev{EnergyGNN did not reach the solver tolerance on Coordinates for any of
the five parameter-initialization seeds. We therefore report no predictive
score and mark the entry with $\dagger$. The available non-converged iterates
performed substantially worse than the other models, but they are not treated
as valid equilibrium evaluations.}

\paragraph{EIGNN coverage.}
EIGNN is reported only on Counting. Its implementation performs a
dense eigendecomposition of the complete graph object supplied to the forward
pass, with $O(n^3)$ time and $O(n^2)$ memory. Sums (100{,}000 nodes) and MNIST
Terrain (approximately 148{,}000 nodes) exceed the configured node-count
safety guard. EIGNN runs were also attempted on Coordinates, but none produced
a valid completed result; the available logs do not establish the cause. The
``---'' entries therefore denote configurations without a valid evaluation
result, rather than reported predictive or convergence measurements.

\subsection{Protocol for Figure~\ref{fig:equilibrium-robustness}}
\label{app:equilibrium-protocol}

We evaluate six \rev{initial-state conditions} that vary in scale and direction.
The default condition is the row-wise $\ell_2$-normalized all-ones state used
during training. The other five are positive random states
\[
H^{(0)}=s\,\operatorname{norm}_p(U),
\qquad
U_{ij}\sim\operatorname{LogNormal}(0,1),
\qquad
s\in\{10^{-3},10^{-1},1,10,10^3\},
\]
where $\operatorname{norm}_p$ acts row-wise. For each model, a trajectory from
the default \rev{initial state} is continued to \(K_{\mathrm{ref}}=100\), defining
the \rev{numerical reference state and prediction}
\[
H^{\mathrm{ref}}
=
H_{\mathrm{default}}^{(K_{\mathrm{ref}})},
\qquad
\widehat Y^{\mathrm{ref}}
=
\widehat Y_{\mathrm{default}}^{(K_{\mathrm{ref}})}.
\]

Let \(H_j^{(k)}\) denote the recurrent state at iteration \(k\) for
\rev{initial-state condition} \(j\). We record
\[
\rho_{j,k}
=
\max_{v\in V}
\frac{\|\bigl(H_j^{(k+1)}\bigr)_v-\bigl(H_j^{(k)}\bigr)_v\|_\infty}
     {\|\bigl(H_j^{(k)}\bigr)_v\|_\infty+\varepsilon},
\]
together with the state and prediction discrepancies
\[
d_{j,k}^{\mathrm{state}}
=
\max_{v\in V}
\frac{\|\bigl(H_j^{(k)}\bigr)_v-\bigl(H^{\mathrm{ref}}\bigr)_v\|_\infty}
     {\|\bigl(H^{\mathrm{ref}}\bigr)_v\|_\infty+\varepsilon},
\qquad
d_{j,k}^{\mathrm{pred}}
=
\frac{\|\widehat Y_j^{(k)}-\widehat Y^{\mathrm{ref}}\|_\infty}
     {\|\widehat Y^{\mathrm{ref}}\|_\infty+\varepsilon}.
\]
The plotted curves are the medians across \rev{initial-state conditions},
\[
\rho_k=\operatorname*{median}_j\rho_{j,k},\qquad
d_k^{\mathrm{state}}
=\operatorname*{median}_j d_{j,k}^{\mathrm{state}},\qquad
d_k^{\mathrm{pred}}
=\operatorname*{median}_j d_{j,k}^{\mathrm{pred}},
\]
and the shaded regions show their interquartile ranges. For the default
trajectory,
\(d_{\mathrm{default},K_{\mathrm{ref}}}^{\mathrm{pred}}=0\) by construction;
the \rev{numerical reference state is not assumed to be an exact equilibrium}.

The Loop Transformer is evaluated by repeatedly applying its learned loop
operator beyond \(K_{\mathrm{train}}=20\). Its one-step logit update is used
in panel~(a); no state-discrepancy curve is reported because its recurrent
state is not exposed. All runs use inference mode with dropout disabled.

\begin{figure*}[t]
\centering
\includegraphics[width=\linewidth]{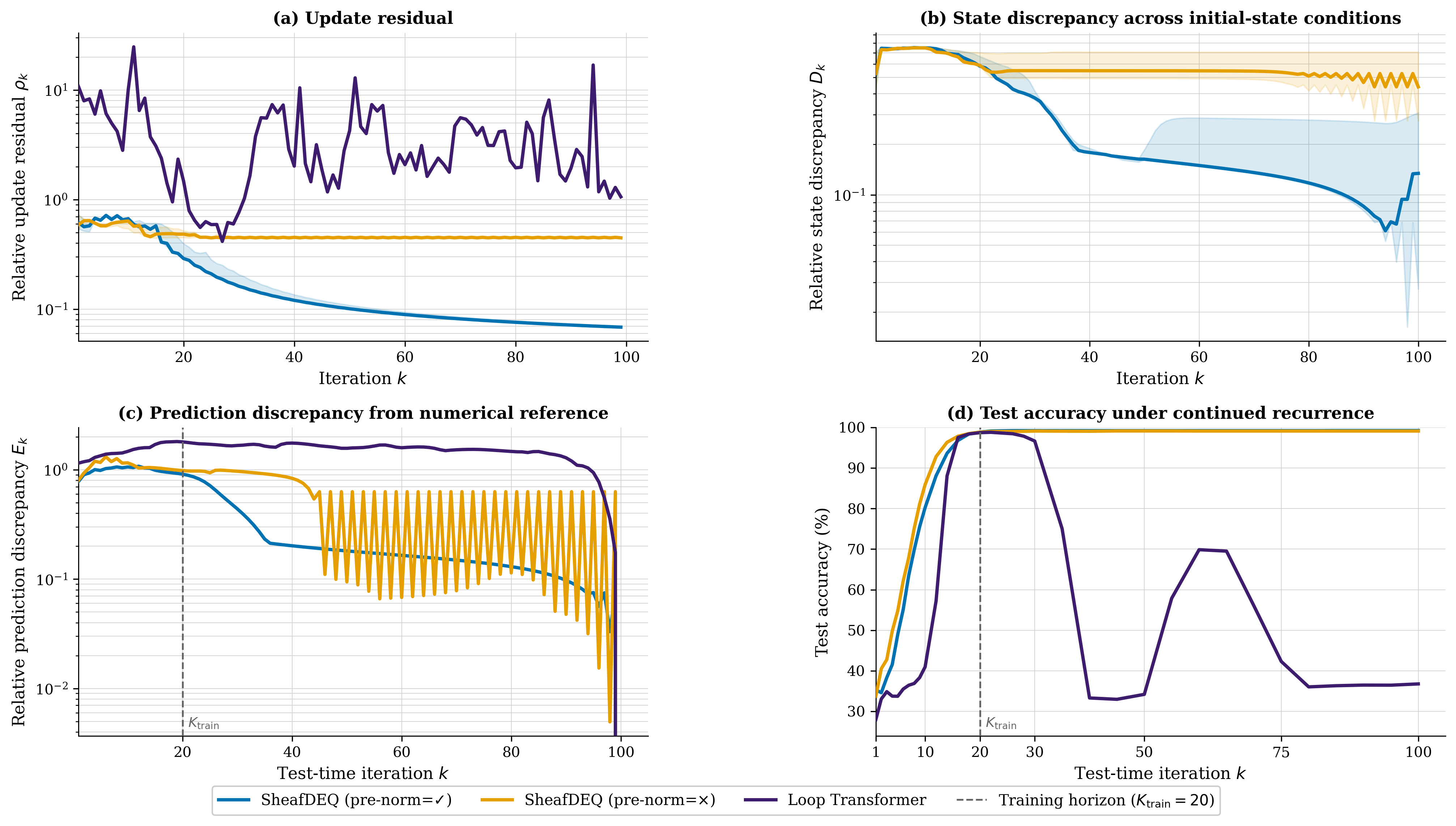}
\caption{\rev{\textbf{Complete finite-iteration recurrence diagnostics.}
Panels show (a) the median relative update residual, (b) the median state
discrepancy from the numerical reference state at \(K_{\mathrm{ref}}=100\), (c) the
median prediction discrepancy, and (d) test accuracy under continued
recurrence. For the Loop Transformer, panel~(a) uses one-step logit updates
and panel~(b) is unavailable. Solid curves and shading show medians and
interquartile ranges across six initial-state conditions; the dashed line
marks its training horizon \(K_{\mathrm{train}}=20\). The references are defined
at the finite iteration \(K_{\mathrm{ref}}\), and the default prediction discrepancy is zero at
\(K_{\mathrm{ref}}\) by construction.}}
\label{fig:equilibrium-robustness-full}
\end{figure*}

\subsection{Bounded-staleness evaluation details}
\label{app:staleness-experiment}

\paragraph{Protocol.}
We evaluate a single trained checkpoint for each \rev{evaluated model} on the fixed
$L_{\mathrm{hetero}}=7$ community-detection graph:
pre-restriction-normalized SheafDEQ and its
ablation without pre-restriction normalization use seed \(43\), while the Loop
Transformer control uses seed \(42\). These are the lowest-index checkpoints
achieving at least $80\%$ validation accuracy at the trained inference
horizon, under the common selection criterion stated in
Appendix~\ref{app:compute}. All schedules share the same graph, input,
parameters, and \rev{initial state}. At logical slot \(t\), every node is
activated at least once in each $s_{\mathrm{upd}}$-slot window and each
neighbour timestamp is sampled from
$\{t-\tau_{\max},\ldots,t\}$. We test
$(s_{\mathrm{upd}},\tau_{\max})\in\{(1,0),(2,1),(5,2),(8,4)\}$ and draw 20 independent schedule
realizations per condition. States are recorded after complete slot commits,
and work is reported as mean local updates per node. Under
$(s_{\mathrm{upd}},\tau_{\max})=(1,0)$,
all schedules coincide and reproduce the synchronous Picard control.

\paragraph{References and summaries.}
For pre-restriction-normalized SheafDEQ, the synchronous reference is the
\rev{Picard iterate after 420 iterations}, whose \rev{state} residual is
\(8.95\times10^{-7}\). The ablation without pre-restriction normalization \rev{does not settle within
500 iterations} (residual \(1.04\times10^{-2}\)); its \(K=500\) synchronous
iterate is therefore used only as a \rev{numerical reference}.  We report the median
and minimum--maximum range of the relative Frobenius discrepancy
\[
d_{\mathrm{out}}(u)
=
\frac{\left\|\widehat Y^{\mathrm{async}}(u)
-\widehat Y^{\mathrm{sync}}\right\|_{\mathrm F}}
     {\left\|\widehat Y^{\mathrm{sync}}\right\|_{\mathrm F}+\varepsilon}
\]
across schedules, rather than cross-seed statistics. We
omit the \((1,0)\) control from the \rev{summary} table because deviations beyond
\rev{420 iterations} reflect \rev{comparison with a finite numerical reference} rather than delayed
communication.

\begin{table*}[t]
\centering
\small
\caption{\textbf{Complete delayed-schedule \rev{results}.}
Each entry is the median \([\min,\max]\) relative Frobenius logit discrepancy
across 20 schedules. \rev{Both SheafDEQ models} use \(u=500\) mean local
updates per node. Loop Transformer is evaluated for its native 20 loops
and is reported only as a finite-horizon control.}
\label{tab:async-full}
\renewcommand{\arraystretch}{1.10}
\begin{tabular}{@{}llc@{}}
\toprule
Model & $(s_{\mathrm{upd}},\tau_{\max})$ & \(d_{\mathrm{out}}\) \\
\midrule
\multicolumn{3}{@{}l}{\textit{Pre-restriction-normalized SheafDEQ
\rev{420-iteration numerical reference}, residual \(8.95\times10^{-7}\)}} \\
\quad SheafDEQ
& \((2,1)\)
& \(7.08\times10^{-4};\ [7.08,7.08]\times10^{-4}\) \\
\quad SheafDEQ
& \((5,2)\)
& \(7.08\times10^{-4};\ [7.08,7.09]\times10^{-4}\) \\
\quad SheafDEQ
& \((8,4)\)
& \(7.08\times10^{-4};\ [7.08,7.09]\times10^{-4}\) \\
\addlinespace
\multicolumn{3}{@{}l}{\textit{Without pre-restriction normalization
\(K=500\) \rev{numerical reference}; residual \(1.04\times10^{-2}\)}} \\
\quad No normalization
& \((2,1)\)
& \(4.00\times10^{-2};\ [7.02\times10^{-3},4.76\times10^{-2}]\) \\
\quad No normalization
& \((5,2)\)
& \(4.12\times10^{-2};\ [1.57\times10^{-2},5.32\times10^{-2}]\) \\
\quad No normalization
& \((8,4)\)
& \(4.08\times10^{-2};\ [1.60\times10^{-2},5.44\times10^{-2}]\) \\
\addlinespace
\multicolumn{3}{@{}l}{\textit{Finite-horizon control}} \\
\quad Loop Transformer (20 loops)
& \((5,2)\)
& \(7.09\times10^{-2};\ [6.32\times10^{-2},7.93\times10^{-2}]\) \\
\bottomrule
\end{tabular}
\end{table*}

The \rev{reported discrepancies for pre-restriction-normalized SheafDEQ} under $(2,1)$, $(5,2)$, and $(8,4)$ differ by
only \(3\times10^{-7}\) across conditions and schedules. \rev{At the available
iteration budget and relative to the finite numerical reference}, this is consistent with a common
numerical-reference floor rather than a resolved staleness trend.
In contrast, the ablation without pre-restriction normalization has discrepancies near
\(4\times10^{-2}\) with materially wider schedule ranges. The Loop
Transformer result is descriptive only: it uses a fixed 20-loop inference
budget and its stale-attention update is an approximate simulation of delayed
neighbour information,
so it is not a matched equilibrium comparison.

\begin{figure*}[t]
\centering
\includegraphics[width=\linewidth]{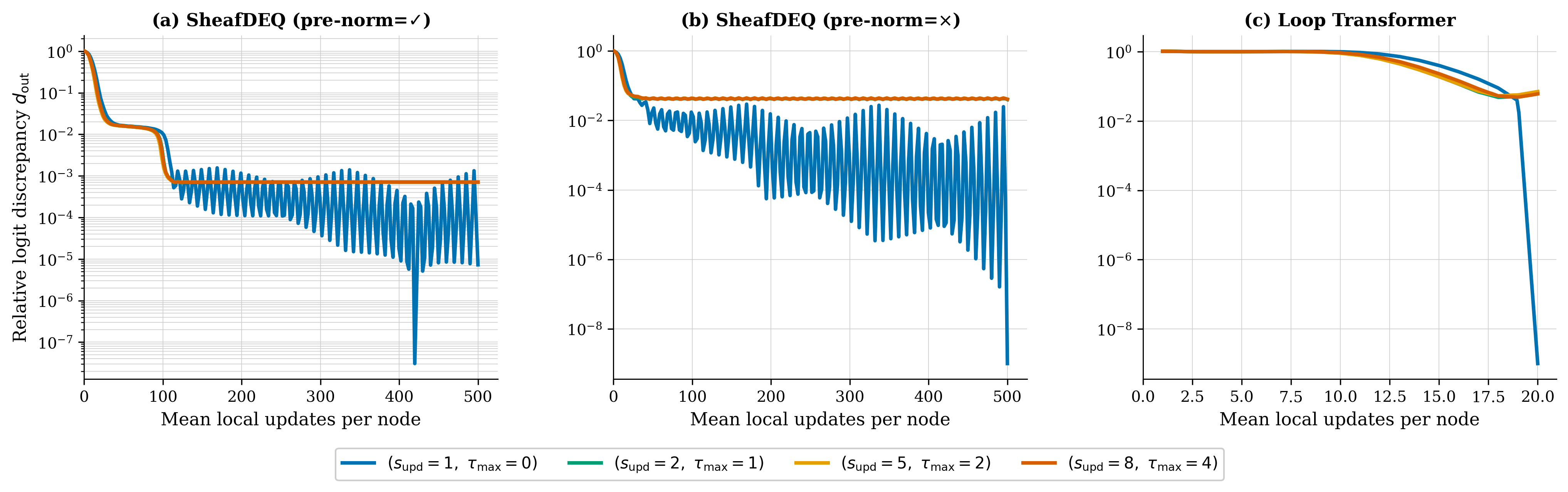}
\caption{\textbf{Post-slot asynchronous inference trajectories.}
Relative Frobenius logit discrepancy from the synchronous reference versus
mean local updates per node.  Thin curves denote the 20 sampled schedules;
thick curves show their median.  All points are recorded after complete
logical-slot commits, at their observed work coordinates, without
interpolation. For pre-restriction-normalized SheafDEQ, the delayed conditions
\((2,1)\), \((5,2)\), and \((8,4)\) reach nearly identical
\rev{discrepancies at the evaluation budget}.  The ablation without pre-restriction normalization displays substantially larger
discrepancies and schedule variation.  The \((1,0)\) synchronous control is
\rev{shown only through iteration 420}: later discrepancies reflect
comparison with a finite numerical reference, rather than delayed
communication.}
\label{fig:async-trajectories}
\end{figure*}

\section*{GenAI Usage Statement}
Generative AI tools were used only for language editing and basic coding assistance.

\end{document}